\documentclass[conference]{IEEEtran}
\IEEEoverridecommandlockouts
\usepackage{cite}
\usepackage{amsmath,amssymb,amsfonts,amsthm}
\usepackage{algorithmic}
\usepackage{graphicx}
\usepackage{textcomp}
\usepackage{xcolor}
\usepackage{subfigure}
\usepackage{multirow}
\usepackage[linesnumbered,ruled,vlined]{algorithm2e}
\usepackage{booktabs}
\usepackage{url}
\usepackage{dirtytalk}
\SetKwInput{KwInput}{Input}                
\SetKwInput{KwOutput}{Output}  
\usepackage{caption}
\newcommand{\eat}[1]{}
\newtheorem{myDef}{Definition}
\newtheorem{theorem}{Theorem}

\def\BibTeX{{\rm B\kern-.05em{\sc i\kern-.025em b}\kern-.08em
    T\kern-.1667em\lower.7ex\hbox{E}\kern-.125emX}}
\begin{document}

\title{Online Anomaly Detection over Live Social Video Streaming\\
\thanks{*Chengkun He and Xiangmin Zhou contributed equally to this research.}}
\author{
\IEEEauthorblockN{
Chengkun He\IEEEauthorrefmark{1,2}, 
Xiangmin Zhou\IEEEauthorrefmark{1,*}, 
Chen Wang\IEEEauthorrefmark{2},
Iqbal Gondal\IEEEauthorrefmark{1},
Jie Shao\IEEEauthorrefmark{3},
Xun Yi\IEEEauthorrefmark{1}, 
}
\IEEEauthorblockA{\IEEEauthorrefmark{1}School of Computing Technologies, RMIT University, Melbourne, Australia}
\IEEEauthorblockA{\IEEEauthorrefmark{2} Data61 CSIRO, Australia; \IEEEauthorrefmark{3}University of Electronic Science and Technology of China}
\IEEEauthorblockA{Email: 
chengkun.he@student.rmit.edu.au,
\{xiangmin.zhou,iqbal.gondal,xun.yi\}@rmit.edu.au,\\ chen.wang@data61.csiro.au,  shaojie@uestc.edu.cn}
}
\eat{
\author{
\IEEEauthorblockN{Chengkun He}
\IEEEauthorblockA{
    \textit{RMIT University, Data61, Australia}\\
    chengkun.he@student.rmit.edu.au
}
\and
\IEEEauthorblockN{Xiangmin Zhou}
\IEEEauthorblockA{\textit{RMIT University, Australia} \\
xiangmin.zhou@rmit.edu.au
}
\and
\IEEEauthorblockN{Chen Wang}
\IEEEauthorblockA{\textit{Data61, Australia} \\
chen.wang@data61.csiro.au}
\and
\IEEEauthorblockN{Iqbal Gondal}
\IEEEauthorblockA{\textit{RMIT University, Australia} \\
iqbal.gondal@rmit.edu.au
}
\and
\IEEEauthorblockN{Jie Shao}
\IEEEauthorblockA{\textit{University of Electronic Science and Technology of China
} \\
shaojie@uestc.edu.cn}
\and
\IEEEauthorblockN{Xun Yi}
\IEEEauthorblockA{\textit{RMIT University, Australia} \\
xun.yi@student.rmit.edu.au}
}}

\maketitle

\begin{abstract}
Social video anomaly is an observation in video streams that does not conform to a common pattern of dataset's behaviour. Social video anomaly detection plays a critical role in applications from e-commerce to e-learning. Traditionally, anomaly detection techniques are applied to find anomalies in video broadcasting. However, they neglect the live social video streams which contain interactive talk, speech, or lecture with audience. In this paper, we propose a generic framework for effectively online detecting Anomalies Over social Video LIve Streaming (AOVLIS). Specifically, we propose a novel deep neural network model called Coupling Long Short-Term Memory (CLSTM) that adaptively captures the history behaviours of the presenters and audience, and their mutual interactions to predict their behaviour at next time point over streams. Then we well integrate the CLSTM with a decoder layer, and propose a new reconstruction error-based scoring function $RE_{IA}$ to calculate the anomaly score of each video segment for anomaly detection. After that, we propose a novel model update scheme that incrementally maintains CLSTM and decoder. Moreover, we design a novel upper bound and ADaptive Optimisation Strategy (ADOS) for improving the efficiency of our solution. Extensive experiments are conducted to prove the superiority of AOVLIS.
\end{abstract}


\section{Introduction}
With the popularity of smartphones and 5G networks, online social video live streaming service platforms have become popular outlets for people's entertainment, and key marketing and communication tools that help brands reach their online audience. Individuals access these platforms for playing video games and interacting with other users. Online sellers showcase their products through these streaming platforms, interacting with audience to encourage customer purchases. 
\eat{According to the statistics .... how much revenues have been gained due to the livestreaming..... 
The showcase is promoted by sellers' attractive actions (this is anormaly). ... we can have Fig.1 here. Then we can say, what is anomaly... 
Therefore, an appropriate anomaly detection is a promising way of increasing the audience interaction, enhancing the effect of online sells promotion. }
According to the Benchmark Report \cite{benchmark_report}, the global influencer marketing market size has more than doubled since 2019, reaching a record of 21.1 billion U.S. dollars in 2023. Due to the high level engagement, audience prefer to interact with influencers, liking, commenting and sharing, which helps businesses reach a wider audience and increase brand awareness. During a showcase, sellers could attract the audience to promote products by attractive actions. 
An example is shown in Fig. \ref{fig:influencer_show}. The influencer introduces the product in the normal case (Fig. \ref{fig:influencer_show_1}). Then, he wobbles the balance
board, triggering more bullet comments shown in Fig. \ref{fig:influencer_show_2}. When an influencer performs captivating actions and receives significant audience interactions, we consider this case as an anomaly.\eat{Appropriate anomaly detection is a promising way of increasing audience interaction and enhancing the effect of online sales promotion.}
\eat{According to the Benchmark Report\footnote{https://influencermarketinghub.com/}, the global influencer marketing market size has more than doubled since 2019, reaching a record  21.1 billion U.S. dollars in 2023. Due to the high level of engagement, audiences prefer to interact with influencers, liking commenting and sharing, which helps businesses reach a wider audience and increase brand awareness. During a showcase, sellers could attract the audience to promote products by attractive actions. 
An example is shown in Fig. \ref{fig:influencer_show}. The influencer introduces the product in the normal case (Fig. \ref{fig:influencer_show_1}). Then, he wobbles the balance
board, triggering more bullet comments shown in Fig. \ref{fig:influencer_show_2}. When an influencer performs captivating actions and receives significant audience interactions, an anomaly happens, stimulating the high level audience interaction and further purchase behaviour.
Therefore, appropriate anomaly detection is a promising way of increasing audience interaction and enhancing the effect of online sales promotion.}As an observation deviated from the normal behaviour over social live video streams, the appearance of anomaly often indicates the happening of notable influencer or audience behaviours, and the high interactions between speakers and audience. For example, in live commerce, an anomaly may indicate the happening of a soft advertisement or big online purchasing activities. Monitoring anomalies over social live video streams has many applications such as online product promotion and e-learning. Consider an application of live video anomaly monitoring in online product promotion. When a company creates an influencer video to promote a brand, the influencer would make the presentation in some specific styles, emphasize the product features by intriguing and attractive actions to stimulate audience's attention and interactions. \eat{As shown in Fig. \ref{fig:influencer_show}, the influencer introduces the product in the normal case (Fig. \ref{fig:influencer_show_1}). He wobbles the balance board to attract more bullet comments, triggering an anomaly shown in Fig. \ref{fig:influencer_show_2}. }Detecting the anomalies is helpful for companies to analyze user feedback, improving their live video production planning and their products, and boosting the sales of their brand. Thus, how to effectively monitor anomalies over social live video streams has become a vital research problem. 

\begin{figure}[htb]
\centering\vspace{-2ex}
    \subfigure[Normality]{
    \includegraphics[width=3.02 cm]{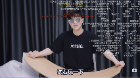}
    \label{fig:influencer_show_1}
    }\vspace{-2ex}
    \subfigure[Anomaly]{
    \includegraphics[width=3.02 cm]{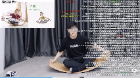}

    \label{fig:influencer_show_2}
    }
    \caption{\small {An example of real-life anomaly over live social video stream.
    }}\vspace{-2ex}
    \label{fig:influencer_show}
\end{figure}

This paper focuses on the problem of online anomaly detection over social video live streams, where audience can interact with the presenters or influencers in videos. This problem requires a meaningful way of justifying a social video anomaly. Traditional techniques \cite{DBLP:conf/cvpr/SultaniCS18,DBLP:conf/mm/YuWCZXYK20,DBLP:journals/tmm/LiCL21} identify the anomalies based on the video content only. However, in online social video live streaming, the special actions of influencers or presenters are just the behaviours of themselves in virtual world, which cannot be connected to any meaningful real world events without the reactions to their behaviours. Thus, in real social live video applications, it is vital to note that the anomalies are related to both the actions of influencers and the reactions of audience, where the anomalies cannot be properly justified without the audience participation. Accordingly, for such a new live media, traditional anomaly detection techniques cannot work properly. Unlike traditional anomaly detection using real world events related videos, our social anomaly detection takes social live video streaming as the data source, and returns a list of anomalies that are triggered by the video influencers and reflected by audience behaviours. Technically, we can extend the traditional anomaly detection techniques for our social video anomaly identification. However, these traditional techniques did not consider the interactions with audience and rely on the actions of speakers, thus they cannot recognize the anomaly of social videos. 

Motivated by the limitations of traditional anomaly detection techniques, this paper addresses the anomaly detection problem for social live videos by fully exploiting the temporal audience interactions and the mutual influence between audience and presenters for anomaly detection over social video live streams. Social interactions have been proven to be effective for behaviour prediction in crowd anomaly detection \cite{Lin2019SocialMI}. However, the behaviour interaction in \cite{Lin2019SocialMI} refers to the crowd interaction in videos, which is the influence among multiple presenters in live video streams. This approach works based on an assumption that audience feedback cannot be received by presenters online or the presenters simply ignore audience's feedback during the video streaming, which is not practical. The main challenges of the social interaction-based anomaly detection are the temporal influence of audience behaviours to themselves and the mutual influence between presenters and audience as two parties in social video live streams, which may affect their behaviours. A practical anomaly detection for social live videos should well recognize the temporal influence of each party and the mutual influence of two parties. 

This paper proposes a generic live social video anomaly detection framework called AOVLIS by exploring the temporal interaction of audience and the mutual influence between presenters and audience. Specifically, we first propose a novel CLSTM model to capture each influencer's behaviour history, the reaction behaviours of audience to influencer behaviour, and the mutual interactions between influencers and audience over a temporal ordered history video segment series for predicting the behaviours of influencers and audience at the next time point in a hidden latent space. To measure the anomaly score of video segment, we apply an encoder layer to map these predicted behaviours into the original feature space, and propose a reconstruction error-based scoring function that considers the information loss from the behaviour prediction of both influencers and audience interactions. This work advances the anomaly detection research in two folds: 1) this is the first work to study the problem of anomaly detection from live social video streams, while traditional techniques focus on general video streams like surveillance videos; 2) unlike traditional deep neural network models for sequential data \cite{6302929,DBLP:journals/neco/HochreiterS97,DBLP:journals/corr/abs-1810-04805} that treat multiple incoming streams independently, CLSTM includes two interactive layers, each of which captures the temporary dependency of its stream and the social dependency on the other layer, thus more practical and extendible for modelling multiple streams with mutual interactions. We summarise our contributions as follows.

\begin{itemize}
    \item We are the first to realize that the anomalies in live social video streaming are related to both the actions of influencers and the reactions of audience, thus reconstruction error scoring (RE) over visual features is inappropriate.
    \eat{We are the first to realize that the assumption that ``the anomalies are related to real world events" is not true in applications, thus reconstruction error scoring (RE) over visual features is inappropriate.}\eat{We are the first to recognize that the assumption that ``anomalies could not be directly associated with real-world events" holds true in this context. More precisely, anomalies are connected to a series of actions and reactions in a virtual world. Consequently, depending solely on reconstruction error scoring (RE) based on visual features is inappropriate.}We propose a weighted RE scoring ($RE_{IA}$) over the temporal behaviours of influencers and that of audience interactions to decide the anomaly score of a segment, which is consistent with the characteristics of live videos.
    
    \item We propose a novel deep neural network model CLSTM to predict the behaviours of influencers and that of audience interactions at the next time point. CLSTM well captures the temporal historical behaviour of the influencers and audience interactions over video sequences and the mutual temporal dependency of influencer behaviours and audience interaction behaviours over each other. 

    \item We present an effective adaptive dimension group representation (ADG) to reduce the dimensionality of the video segment features and design a distance measure in the reduced space to filter out the anomaly candidates. We further design an adaptive optimisation strategy to reduce the filtering cost in anomaly detection.

    \item We propose a novel dynamic updating strategy to incrementally update our model for steady streams. We collect four real video datasets and manually label anomalies in them to evaluate the high performance of our framework.

\end{itemize}

The rest of this paper is organized as follows. Section \ref{sec:related-work} provides an overview of related work. Section \ref{sec:pd} formulates the problem. Section \ref{sec:CLSTM} presents the details of our CLSTM for anomaly detection and the incremental updating algorithm, followed by our anomaly identification optimization scheme in Section \ref{sec:opti}. We report the experimental evaluation results in Section \ref{sec:exp}, and conclude the whole work in \ref{sec:conclusion}.

\section{Related Work}\label{sec:related-work}

Anomaly detection has been studied in different fields such as transportation \cite{DBLP:conf/icde/Liu0CB20}, electronic IT security \cite{DBLP:conf/sigmod/YeLHLS21} and industrial systems \cite{GOMEZGONZALEZ2016484}. Video anomaly detection includes metric-based and learning-based methods. Early researches on this topic rely on handcrafted appearance and motion feature for the detection \cite{DBLP:journals/pr/CongYL13,DBLP:conf/iccv/LuSJ13,DBLP:journals/pami/LiMV14}. These methods incur low efficiency for large-size data and lacking robustness to the data changes.

Deep learning has been used for video anomaly detection ~\cite{DBLP:conf/sitis/KimC19,DBLP:conf/ijcnn/GausBAGBB19,DBLP:conf/ijcai/LiuLLZG19,DBLP:conf/iccv/XiaCWHS15,DBLP:conf/cvpr/LiuLLG18,DBLP:conf/icde/Liu0CB20,DBLP:conf/cvpr/DoshiY20b,DBLP:conf/cvpr/0003CNRD16,DBLP:conf/iccv/HinamiMS17,DBLP:conf/cvpr/IonescuKG019,DBLP:conf/mm/YuWCZXYK20}. 
\eat{Generally, it is believed that abnormal instances are various and occur regularly, while patterns of normal data are limited. 
Hence, a latent space or distribution of normal data is learned for measuring the ``distance" between given instance and learned patterns and deciding the anomaly.}
DARE learns reconstructions by autoencoder in an unsupervised manner \cite{DBLP:conf/iccv/XiaCWHS15}. \eat{It revealed that when data are embedded into low-dimensional representations and then reconstructed by an autoencoder, the inliers tend to have smaller reconstruction errors than the outliers.} AnoGAN concurrently trains a generative model and a discriminator to enable the anomaly detection on unseen data in an unsupervised way \cite{DBLP:conf/ipmi/SchleglSWSL17}.
Modified GANs concurrently learn an encoder during training for anomaly detection \cite{DBLP:conf/iclr/ZenatiFLMC2018}.

Supervised learning-based anomaly detection has been studied \cite{DBLP:conf/sitis/KimC19}. Dual-CNN isolated liquid and electrical objects by type and screen them for abnormalities \cite{DBLP:conf/ijcnn/GausBAGBB19}. GAN addressed the labeled data shortage by transfer learning and learns the video anomaly detection model in a supervised way \cite{DBLP:conf/ideal/ShinC18}. MLEP detects anomaly by enlarging the margin between normal and abnormal events~\cite{DBLP:conf/ijcai/LiuLLZG19}.
Weakly supervised anomaly detection introduced multi-instance learning (MIL) over both normal and abnormal videos \cite{DBLP:conf/icip/ZhangQM19}. However, it trains the anomaly detection model in a supervised way, and incurs several problems such as data imbalance and vulnerability to undefined anomaly~\cite{DBLP:conf/sitis/KimC19}. RTFM \cite{DBLP:conf/iccv/TianPCSVC21} improves MIL-based anomaly classification using temporal feature magnitude to learn more discriminative features.
However, they all assume normal events can be well predicted, which cannot be adopted in our work. Recent work proposed interactive outlier detection and train the model in unsupervised manner for novelty detection \cite{DBLP:conf/er/HeZW22}. However, it handles videos that do not have real-time audience interaction. When these methods are applied to social live video streams, they assume that audience feedback cannot be received by presenters online or the presenters simply ignore audience's feedback during the video streaming, which is not true in practice and thus leads to low detection quality. Thus they cannot well detect the anomalies over streaming. 

Multi-modal methods have been proposed based on both metrics \cite{DBLP:journals/pr/CongYL13} and learning \cite{DBLP:conf/cvpr/DoshiY20b,DBLP:conf/cvpr/LiuLLG18} to detect anomalies in videos. Sparse reconstruction cost detects anomalies based on the labeled local and global normal events \cite{DBLP:journals/pr/CongYL13}. 
Transfer learning and any-shot learning were exploited for anomaly detection \cite{DBLP:conf/cvpr/DoshiY20b}. A future frame prediction network was proposed based on U-Net for anomaly detection \cite{DBLP:conf/cvpr/LiuLLG18}. However, all these multi-modal methods focus on anomalies with big visual differences from normal events, thus, cannot detect anomalies when abnormal and normal events are visually similar. Also, they process different modalities separately as the model inputs, which ignores the interactions between multiple modalities. 
This work addresses the problem of effective and efficient anomaly detection over social video live streams by fully exploiting the mutual influence between influencers and audience, which better reflects the characteristics of social live videos and thus enables high-quality anomaly prediction. 
In existing work, user interactions refer to the crowd interaction in videos. In this work, audience interaction refers to the reaction of users who watch the videos. Unlike LSTM, CLSTM is a novel multi-modal model that embeds the audience interaction and visual features and their mutual interactions. Streaming companies like Netflix detect anomalies over streaming behaviours that are the interactions in data, which is different from the interactions between the video and audience outside of the video. We have solved a new problem with a new solution. \eat{The notations used in this paper are listed in Table \ref{fbl-notations}.

\begin{table}
	\centering
	\caption{\small{Notation Table.}}\label{fbl-notations}
\setlength{\tabcolsep}{0.5ex}
\begin{tabular}{|l|l|}
\toprule
\hline
Notation & 
	Definition \\
	\midrule
 \hline
$V=\{v_1, v_2,\dots, v_n\}$
	& A series of social video segments.\\
$C=\{c_1, c_2,\dots, c_n\}$
	& A series of audience interactions.\\
$f_i$	& The action feature \\
$d$     & The feature dimension\\
$a_i$   & Audience interaction feature\\	
$I$ & The activity feature set  \\
$A$ & The audience interaction feature set  \\
\bottomrule
\hline
\end{tabular}
\end{table}
}

\section{Problem Formulation}\label{sec:pd}

\begin{figure}[t]
\centerline
{
    \includegraphics[width=0.49\textwidth]{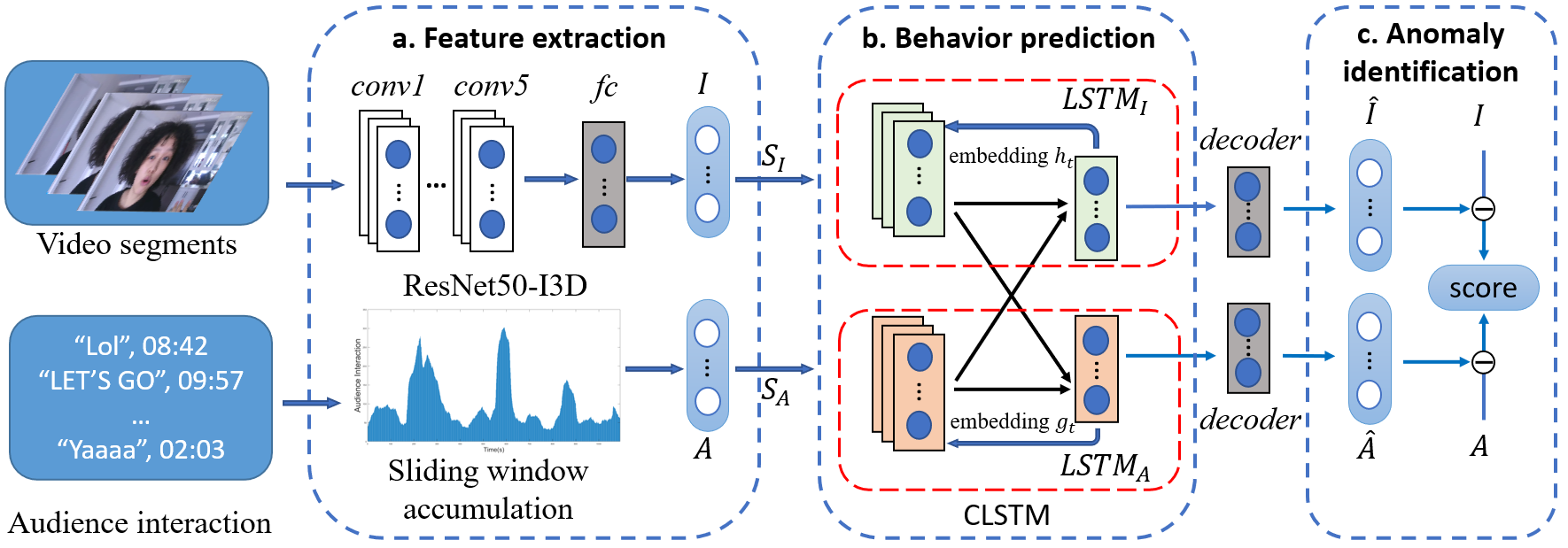}
}
\vspace{-1ex}
\caption{\small{ Overview of AOVLIS.}}
\label{fig:framework_whole}\vspace{-4ex}
\end{figure}

This section formally defines the problem of anomaly detection over social video live streaming. We first define anomaly.

\begin{myDef}
    \label{def_1}
    In social video live streaming, a {\normalfont video~ segment} contains its own content and the audience interactions to it in the forms of bullet comments and live chat. Given a social video live streaming, a series of video segments $\{v_1,\cdots,v_n\}$ can be extracted using a sliding window. An {\normalfont anomaly} is a video segment showing a specific action of speakers, such as speaking and promoting, which has different style from those of the majority parts over the streaming.
\end{myDef}

Here, style refers to the behaviour pattern such as a facial expression or body movement etc. Our problem is proposed for live social video streaming and can be extended to any applications involving video productions with external user participation. AOVLIS is designed for detecting anomalies over live social video streaming with audience interaction. We handle not only a use case, but a new data type with many applications such as e-learning and highlight detection.
A video segment is the series of $q$ temporally ordered adjacent video keyframes. As \cite{DBLP:journals/sensors/LeeCLW19}, we set $q$=10 in the test. A segment in live social video carries three types of features: visual content, temporal information, and real-time comments. Traditional anomaly detection exploits the visual content and temporal information to achieve high performance for general videos involving notable actions or object movements. However, in live social videos that contain interactive talk, speech or lectures, the spatial-temporal features are not discriminative enough due to the limited actions and movement of speakers in them. Fortunately, audience interactions provide other clues about the anomalies from the reflections of audience side. Exploiting audience interactions as a complement of visual features will be promising for the information
compensation in anomaly detection in social video live streams. We extract two types of features to describe a video segment.

\begin{itemize}
\item Action recognition feature: The spatio-temporal representation of each video segment. Each segment $v_i$ (64 frames with resolution $480\times 480$ ) is fed into a pre-trained neural network, ResNet50-I3D \cite{DBLP:conf/cvpr/CarreiraZ17}, to extract a $d-$dimensional ($d=400$ in tests) feature $f_i$. The features of consecutive segments are ordered to form a sequence representation. 
\item Audience interaction feature: 
\eat{The statistics of the real-time comments associated with each video segment. The audience interaction feature of a video segment indicates the number of comments to it, which shows the reaction of audience within a short time slot, the word embedding and sentiment features for each audience interaction.}
The audience interaction feature of a video segment is combined by the number of comments, word embedding, and sentiment embedding for audience interaction within a short time slot.
\end{itemize}

We address the problem of detecting anomalies
from social video live streams with two types of features, as defined below:

\begin{myDef}
    \label{def_3}
   Given a social video live streaming described as a series of video segments $\{v_1,\cdots,v_n\}$ and corresponding  audience interactions $\{c_1,\cdots,c_n\}$, and an anomaly score function $RE_{IA}(t)$, our anomaly detection detects   
   a list of segments with top anomaly scores, $S_{abnormal}$, such that for any clip $v_t\notin S_{abnormal}$ and $v_{t'}\in S_{abnormal}$,
    the following condition holds:
           $RE_{IA}(t)< RE_{IA}(t')$. 
\end{myDef}

$RE_{IA}(t)$ is a weighted function that reflects the reconstruction errors from both the features of $v_t$ and $c_t$. 
\eat{
\begin{figure}[t]
\centering
    \subfigure[bullet comments on bilibili]{
    \includegraphics[width=3.9 cm]{figures/bilibili_chat_new.png}
    \label{live_com:bili}
    }
    \subfigure[live chat on twitch]{
    \includegraphics[width=4.1 cm]{figures/twitch_chat_new.png}
    \label{live_com:twitch}
    }
    \eat{\subfigure[audience interaction over time]{
    \includegraphics[width=3.7 cm]{figures/thresholding.jpg}
    \label{live_com:au_in}}
    }\vspace{-4ex}
    \caption{\small Examples of real-time comments on different platforms.}
    \label{live_com}\vspace{-4ex}
\end{figure}

Fig.~\ref{live_com} show the social video examples with bullet comments and live chat from social platforms. Audience makes real-time comments while watching the videos. Each comment in the red boxes has a timestamp, and highly correlated with the video content at this time point. Our work is to address the problem of effective and efficient anomaly detection over social video live streams. }We propose a framework for detecting Anomalies Over social Video LIve Streaming (AOVLIS), as shown in
Fig. \ref{fig:framework_whole}. AOVLIS includes three main components: feature extraction, behaviour prediction, and anomaly identification. In addition,
we design an incremental model update algorithm to optimize the efficiency of dynamic maintenance. In training phase, we first extract features as shown in Fig. \ref{fig:framework_whole} (a) by applying a neural network ResNet50-I3D to each video segment and quantifying the audience interaction within a time window to it, and organising a representation embedding temporal information from consecutive segments. Then, we predict the behaviour of influencers and that of audience interaction at the next time point based on CLSTM to capture the mutual influence between presenters and audience as shown in Fig. \ref{fig:framework_whole} (b), which are further mapped into the original feature space using a decoder layer. In test phase, we maintain the model over time dynamically. In anomaly detection as shown in Fig. \ref{fig:framework_whole} (c), the $RE_{IA}(t)$ of a current prediction decides if the next segment is abnormal. We propose a novel dimensionality reduction and adaptive optimisation scheme to improve the detection.

\section{CLSTM-based Anomaly Detection}\label{sec:CLSTM}
We first extract the features of each video segment. Then we present our CLSTM to predict the behaviours of influencers and audience interaction, a new scoring function to measure the anomaly score of the predicted video segment, and a dynamic updating strategy for CLSTM maintenance.  

\subsection{Social Live Video Feature Extraction}
\label{sec_dpfe}
To find the anomalies over social video live streams effectively and efficiently, we need to extract compact yet informative features for social live videos. Extracting features from each video frame incurs high computation cost while losing the temporal information dependency of influencer and audience behaviours. To overcome the difficulties of this task, we need to find a number of video segments and extract features from them for the behaviour prediction. As discussed in \cite{DBLP:conf/cvpr/CarreiraZ17}, a video segment of 64 frames contains enough information on an action in its video. Following \cite{DBLP:conf/cvpr/CarreiraZ17}, we divide a social video live stream into a series of 64-frame segments using a sliding window of size 64 with interval size 25 that is the number of frames in $1s$ video segment as in \cite{DBLP:journals/sensors/LeeCLW19}. For each video segment, we extract its action and audience interaction features.

\subsubsection{Action Recognition Feature}
To identify the influencer behaviours in videos, we need to capture spatial-temporal visual information. We use the existing feature selection methods, and select a suitable visual feature for this work.
A typical feature extraction is to describe a video segment as a basic processing unit (e.g. action and event) that is described as handcrafted features like HOG-3D \cite{DBLP:conf/bmvc/KlaserMS08} and iDT \cite{DBLP:conf/iccv/WangS13a}, or combined features like the combination of appearance and motion \cite{DBLP:conf/nips/SimonyanZ14,DBLP:conf/mm/YuWCZXYK20}. However, extracting features is often computationally intensive and intractable for large datasets \cite{DBLP:conf/iccv/TranBFTP15}. Combined features can be extracted using neural networks efficiently for vision related tasks such as detection and classification. However, the preprocessing (e.g. detecting region of interest) is required before the feature extraction and the feature fusion may not be robust for various datasets. Inspired by the successful application of 3D ConvNets feature on action recognition \cite{DBLP:conf/cvpr/CarreiraZ17}, we apply ResNet50-I3D to generate spatio-temporal features from video segments.
Given a video segment series $V =\{v_1,\cdots,v_M\}$, we adopt pre-trained ResNet50-I3D $\Phi_F(\cdot)$ trained on Kinetcis400
\cite{Kinetcis400} as the action feature extractor: $f_i = \Phi_F(v_i)$,
\eat{\begin{equation}
\label{model:ResNet}
    f_i = \Phi_F(v_i)
\end{equation}}where $f_i\in R^{d_1}$ is the action feature and $d_1$ is its dimension.

\subsubsection{Audience Interaction Feature}
Using action recognition features, the behaviours of influencers can be captured to identify the anomalies from influencer side. However, anomalies in social video live streams are reflected not only by the actions of influencers themselves.\eat{However, anomalies in social video live streams are not solely reflected by the actions of influencers themselves.} They are highly related to the reaction of audience to these actions. Thus, it is necessary to extract the audience interaction features to reflect how audience are engaged in a video segment. 
Intuitively, when audiences watch social live videos, they tend to comment on those interesting parts to them. The number of real-time comments at time $t$ indicates the behaviour of audience at $t$, and a bigger number of comments means more audiences are interested in the video part at this moment. Thus, we extract the audience interaction feature by quantifying the audience interactions on each video segment. Since a speaker action could last for a time period and the audience comments to a speaker action could appear over a period as well, we can use the sum aggregation or mean aggregation of comments in a time window to describe how attractive a video segment is. We normalize the audience interaction to [0, 1] space to avoid the side effect of total audience participation on the anomaly recognition. In the normalized space, given a segment, its audience interaction vectors generated by sum aggregation and mean aggregation are same. Thus we select a basic sum aggregation function which introduces low time cost for the feature generation. We accumulate real-time comments generated within a time window $W_s$ as computed by: $D_t = \sum{\hat{d}_i}, \hat{d}_i\in W_s$,
\eat{\begin{equation}\label{equ:Dt}
    D_t = \sum{\hat{d}_i}, \hat{d}_i\in W_s
\end{equation}}
where $W_s = [\hat{d}_{t-s}, ..., \hat{d}_{t+s}]$, and $\hat{d}_t$ is the number of real-time comments at moment $t$.

Given audience interaction $c_i$ of the corresponding segment, we count every $D_t$ for $t$ contained in $c_i$. Then we concatenate them as a $k$-tuple vector $(D_{i1}, \cdots,D_{ik})$, where $k$ is the number of contained moments. Considering the information within consecutive segments, we conjoin $k$-tuples of $c_{i-1}$,  $c_i$, and $c_{i+1}$ as the representation of $c_i$. To further enhance the effect of real-time comments, we extract the word embedding and sentiment as additional features and concatenate them with the $k$-tuple vector for each time point. Specifically, we first capture comments from $k$-tuples vector with python tools\cite{scrapy}. Then, we compute the average word embedding by pre-trained Word2Vec \cite{gensim} and sentiment feature by TextBlob\cite{textblob}. Following previous work\cite{DBLP:conf/cvpr/LiuLLG18}, we concatenate three features together to represent audience interaction $c_i$. This process is defined as $\Phi_D(\cdot)$, and feature extraction of audience interaction is described as: $a_i = \Phi_D(c_i)$, where $a_i\in R^{d_2}$ is the interaction feature and $d_2$ is its dimension. After obtaining features $I \in R^{M\times d_1}$ and $A\in R^{M\times d_2}$, we generate feature sequences that will be fed into the LSTM-based network as detailed in Subsection \ref{sec_clstm}. We utilize consecutive segments to generate sequence data. Each sequence of action features is defined as $s_t = \{x_{t-q}, \cdots,x_{t-1}\}$ and the sequence data $S_I = \{s_{t_1}, \cdots, s_{t_N}\}$, where $q$ is the defined length of sequence and $N < M$. Given the length $q$, input sequence data is $S_X \in R^{N\times q\times d_1}$ and  $S_A \in R^{N\times q\times d_2}$.

\subsection{The CLSTM Model}
\label{sec_clstm}

An important feature of social video live streaming platforms is the influencer and audience engagement. Audience can post comments on video segments instead of watching the video live only. Speakers can receive audiences' comments and adjust their presentations in real time accordingly instead of following a prefixed presentation style only. Thus, in this work, we consider two roles: (i) the producer and (ii) the audience, and two types of influences: (i) the influence of historical behaviours in each role on her own future behaviour; and (ii) the mutual influence between these two roles. Note that ``producer", ``influencer", ``speaker" and ``presenter" are interchangeable in this paper.

\begin{figure}[htb]
\centering\vspace{-3ex}
    \subfigure[]{
    \includegraphics[width=4.02 cm]{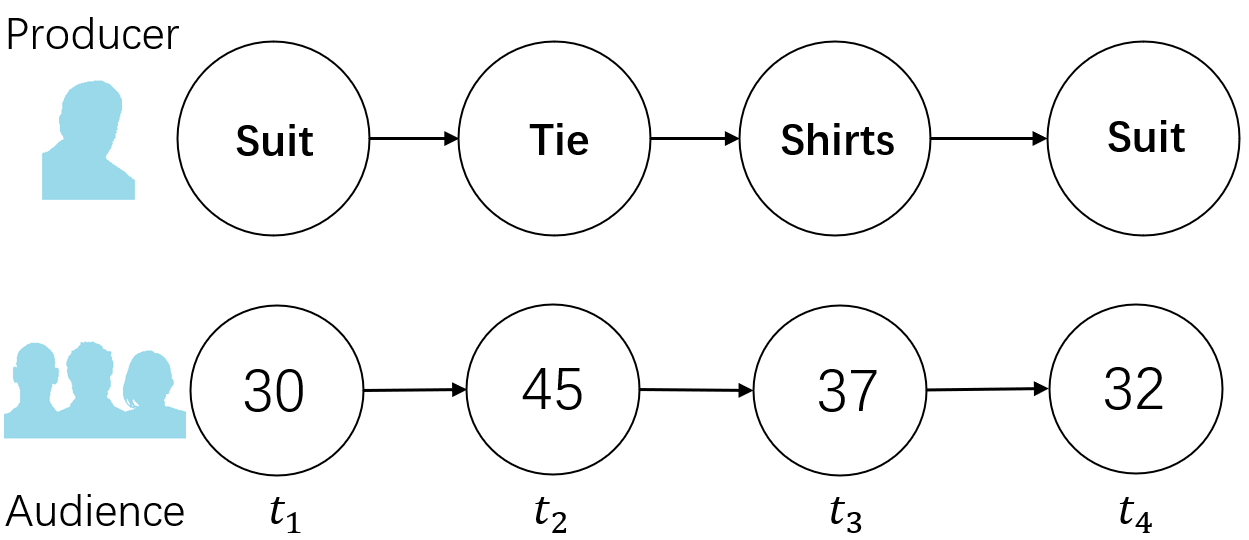}
    \label{fig:scenarios_long_term}
    }\vspace{-2ex}
    \subfigure[]{
    \includegraphics[width=4.02 cm]{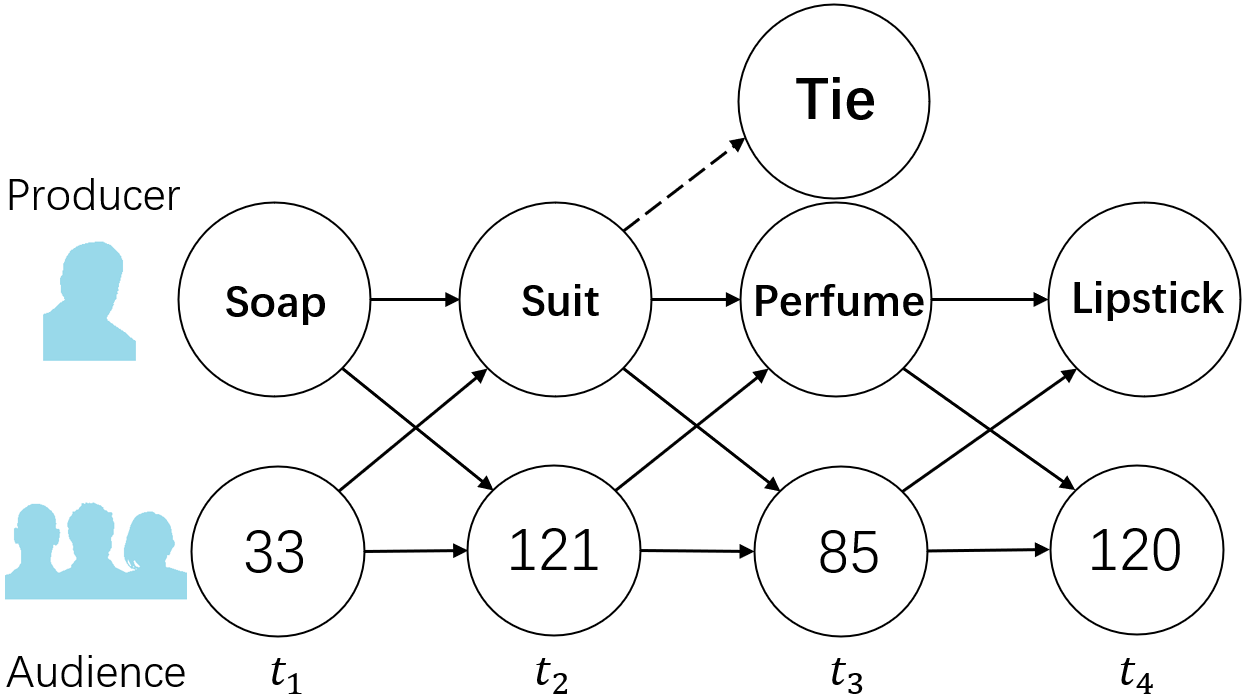}
    \label{fig:scenarios_mutual_inf}
    }
    \caption{\small The application scenario.}\vspace{-2ex}
    \label{fig:scenarios}
\end{figure}

Consider a real scenario shown in Fig. \ref{fig:scenarios}, where audience are watching live streaming created by the producer or influencer. The audience behaviour trajectory contains the number of real-time comments at consecutive time points. The influencer behaviour trajectory contains promoting an item at consecutive time points, which may follow the item pattern \cite{DBLP:conf/icde/ZhouQ0CZ19}, but with adjustment on response to the real time audience feedback in social video live streaming. Fig. \ref{fig:scenarios_long_term} shows the long-term temporal effect of influencer's behaviour and that of audience behaviour on their own future behaviours. Suppose that a fashion influencer is promoting a suit at $t_1$ time point. She would promote a tie matching the suit at $t_2$, and select a shirt matching both the suit and the tie to promote at $t_3$. The item choice for her current promotion is associated with her previous item selections. For audience, the real-time comments at $t_1$ could increase the intention of audience, leading to the increase of real-time comments at the next time points. 

To capture the temporal dependency of behaviours of each role, we need to select a good model that well captures both the long term and short term historical activities of each role. Typical deep neural network models for temporal sequential data include RNN \cite{6302929}, LSTM \cite{DBLP:journals/neco/HochreiterS97}, GRU \cite{DBLP:conf/ssst/ChoMBB14} and BERT~\cite{DBLP:journals/corr/abs-1810-04805}. Among them, RNN feeds back the output of the previous time input to the next time input in the network. This loop structure allows the information to persist in the cycles, and be adopted for sequence data. However, RNN has the problem of vanishing gradients, especially for long data sequences. BERT has been designed for text processing and is very compute-intensive at inference time, which is not suitable for social video live streaming processing. GRU is suitable for fast model training in the case of limited memory consumption, but not good for long sequence data with high accuracy requirement for this work. LSTM is a more advanced RNN that is capable of learning long-term dependencies. Thus, we select LSTM to model the behaviours of influencers and those of audience. The LSTM over influencer behaviours is denoted as $LSTM_I$, and another one over audience behaviour is denoted as $LSTM_A$.

\begin{figure}[hthb]\vspace{-2ex}
\centerline{
    \includegraphics[width=0.32\textwidth]{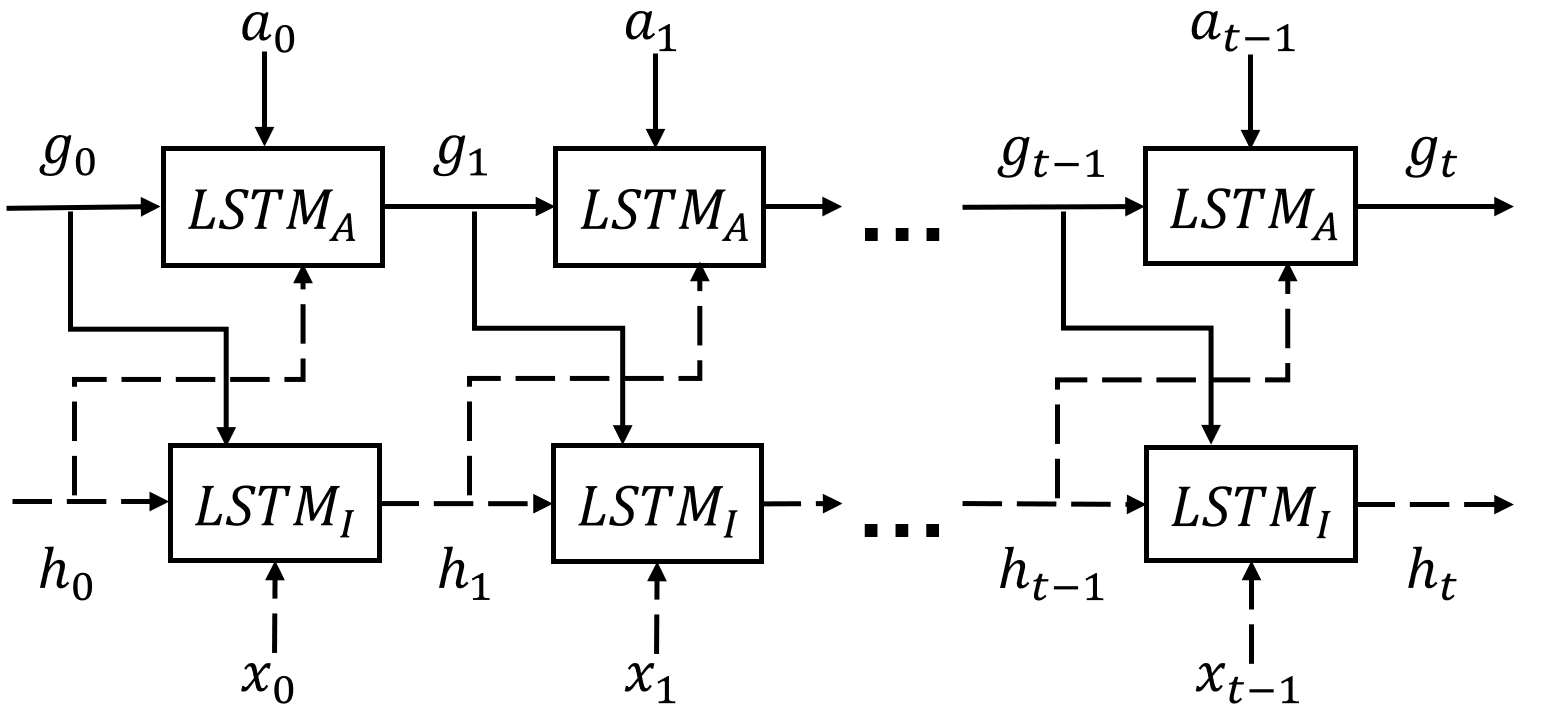}
}\vspace{-1ex}
\caption{\small CLSTM model. }\vspace{-2ex}
\label{fig:LSTM}
\end{figure}

Modelling the influencer behaviour and audience behaviour using two LSTMs separately is not enough due to the mutual influence between influencers and audience. Consider the scenario in Fig. \ref{fig:scenarios_mutual_inf}. An influencer behavioural trajectory may follow the item pattern: ``Soap, Suit, Tie, and others". Likewise, a pattern may also exist in the audience interaction behaviour in the form of abrupt quantity changes of real-time comments in real world live streaming as shown in Fig. \ref{fig:framework_whole} (a). As shown in Fig. \ref{fig:scenarios_mutual_inf}, the influencer promotes soap at time $t_1$, which attracts great attention from audience and accordingly leads to more real-time comments posted by audience at next time $t_2$ considering the possible time delay in comment input. On the other hand, influenced by the strong response to the item ``Soap" from audience, the influencer could adjust her behaviour to promote ``Perfume" instead of the original planned item ``Tie". To capture this mutual influence between influencer and audience, we propose a Coupling LSTM model (CLSTM) as shown in Fig. \ref{fig:LSTM}.

Unlike LSTM that considers the behaviours of influencer only, there are two layers in CLSTM: 
$LSTM_I$ and $LSTM_A$. $LSTM_I$ captures the pattern of normal influencer behaviours by visual spatio-temporal features, while $LSTM_A$ models the corresponding audience interaction. Let $g_t$ and $h_t$ be the hidden states at time $t$ of $LSTM_A$ and $LSTM_I$ respectively. We use arrows to show the flow of audience interaction. For example, if there is an arrow which shows $g_0$ flowing into $LSTM_A$, it means that $g_1$ will contribute to the output, $g_2$ of the $LSTM_A$. Similarly, we use dashed arrows to represent the flow of influencer action information. Considering the mutual influence between speaker action and audience interaction, hidden states will be created by both flows. For example, the speaker gives a special gesture, which attracts great attention from audience and more feedback from them as a consequence. 
Under this situation, the current hidden state $g_t$ of $LSTM_A$ is influenced by previous state $h_{t-1}$ of $LSTM_I$. 
For another, speaker could react to audience interaction. Thus, the hidden state $g_{t-1}$ influences the current hidden state $h_t$.

\subsubsection{The $LSTM_I$ for Modelling Influencer Behaviours} 
The behaviour of an influencer depends on both the historical behaviour trajectory of herself and that of the audience interaction. Thus, we build the $LSTM_I$
by considering both the trajectory of a user’s historical action features and that of her audience interaction features. Just as classic $LSTM$, $LSTM_I$ contains three gates including input gate, forget gate and output gate, and two states including cell state and hidden state. Let $h_{t-1}$ and $g_{t-1}$ be the hidden states of the previous timestamp in $LSTM_I$ and $LSTM_A$ respectively, $f_t$ the current input action recognition feature. The input gate at time $t$ is computed by:
\begin{equation}
    IG_t = \sigma (W_i[h_{t-1},g_{t-1},f_t]+ b_i)
\end{equation}
 where $\sigma$ is the Sigmoid function, $W_i$ is the weight matrix of input $f_t$ associated with hidden state and $b_i$ is a bias vector at $t$ for controlling the importance of the current cell state. Similarly, forget gate is computed by:
 \begin{equation}
    FG_t =  \sigma (W_f[h_{t-1},g_{t-1},f_t]+ b_f)
\end{equation}
where $W_f$ is the weight matrix of input $f_t$ associated with hidden state and $b_f$ is a bias vector for controlling the importance of the previous cell state. In addition, a candidate cell state is generated:
\begin{equation}
    \hat{C}_t = \tanh(W_c[h_{t-1},g_{t-1},f_t]+ b_c)
\end{equation}
where $\tanh$ is the Tanh activation function, $W_c$ is the weight matrix of input $f_t$ associated with hidden state and $b_c$ is a bias vector at $t$ for mapping the input into a representation. 

Having the input gate, forget gate, a candidate cell state at $t$ and the cell state at $t-1$, we can compute the current cell state by $C_t=  IG_t*\hat{C}_t + FG_t*C_{t-1}$. Finally, the output gate is computed by:
\begin{equation}
    OG_t = \sigma (W_o[h_{t-1},g_{t-1},f_t]+ b_o)
\end{equation}
where $W_o$ is the weight matrix of input $f_t$ associated with hidden state and $b_o$ is a bias vector for controlling the current hidden state $h_t = OG_t*\tanh(C_t)$. The weight matrices and biases are optimized by Adam optimizer.
Unlike LSTM, we take hidden state $g_t$ from another $LSTM_A$ into the cell, which utilizes information from audience behaviour. Thus, the $LSTM_I$ can be represented as:
\begin{equation}
    \label{Eq:ht}
    h_t = LSTM_I(h_{t-1}, g_{t-1},f_t)
\end{equation}

\subsubsection{The $LSTM_A$ for modelling audience interaction} 
We first build the $LSTM_A$ to model the behaviour of audience interactions by considering the history audience interaction features and the influence of influencer behaviours. $LSTM_A$ is coupled with $LSTM_I$ in terms of the model structure and the operations. Similar to $LSTM_I$, we compute three gates: input gate, forget gate and output gate, and two states: cell state and hidden state. Unlike $LSTM_I$ that takes $h_{t-1}$, $g_{t-1}$ and $f_t$ as the input of the gates and the cell state, $LSTM_A$ takes $h_{t-1}$, $g_{t-1}$ and $a_t$ as the input of its gates and cell state, where $a_t$ is the current input audience interaction feature. Given $h_{t-1}$, $g_{t-1}$ and $a_t$, we can calculate the input gate at time $t$, the forget gate, the candidate cell state, and the output gate by:
\begin{equation}
IG_t = \sigma (W_i^a[h_{t-1},g_{t-1},a_t]+ b_i^a)
\end{equation}
\begin{equation}
FG_t =  \sigma (W_f^a[h_{t-1},g_{t-1},a_t]+ b_f^a)
\end{equation}
\begin{equation}
    \hat{C}_t = \tanh(W_c^a[h_{t-1},g_{t-1},a_t]+ b_c^a)
\end{equation}
\begin{equation}
  OG_t = \sigma (W_o^a[h_{t-1},g_{t-1},a_t]+ b_o^a)
\end{equation}
Here, $W_i^a$ is the weight matrix of input $a_t$ associated with hidden state and $b_i^a$ is a bias vector at $t$ for controlling the importance of the current cell state. $W_f^a$ is the weight matrix of input $a_t$ associated with hidden state and $b_f^a$ is a bias vector for controlling the importance of the previous cell state. $W_c^a$ is the weight matrix of input $a_t$ associated with hidden state and $b_c^a$ is a bias vector at $t$ for mapping the input into a representation. $W_o^a$ is the weight matrix of input $a_t$ associated with hidden state and $b_o^a$ is a bias vector for controlling the current hidden state $h_t^a = OG_t^a*\tanh(C_t^a)$. $LSTM_A$ can be described as:
\begin{equation}
    \label{Eq:gt}
    g_t = LSTM_A(a_t,h_{t-1},g_{t-1})
\end{equation}

We use the last output $h_t$ and $g_t$ to represent the sequence data. Following that, a decoder is utilized to generate `prediction' $\hat{x}_t$ from $h_t$ and $\hat{a}_t$ from $g_t$. By doing this for each sequence, a set of predictions can be generated. The whole process is represented as: 
\begin{equation}
\label{eq:M}
    (\hat{I},\hat{A}) = M(S_I,S_A,\theta_p)
\end{equation}
where $M(\cdot)$ represents the proposed CLSTM with decoder layers, $\theta_p$ is the set of network parameters.

\subsubsection{CLSTM Training Strategy}
\label{sec:traing_strategy}
Given a video segment at time $t$, we can represent it by action recognition feature $f_t$ and audience interaction feature $a_t$. 
An action sequence $s_I$ is constructed by setting the elements of $s_I$ as the action recognition features and ordering them temporally. 
Similarly, the audience interaction sequence $s_A$ is constructed by setting the sequence elements using the temporally ordered audience interaction features. $s_I$ and $s_A$ are for the same video segment, and have the same length.
Given feature vectors $f$ in the action sequence and $a$ in the audience interaction sequence, we input them into $LSTM_I$ and $LSTM_A$ in CLSTM respectively to generate hidden states $h$ and $g$. The action sequence is converted into a sequence of hidden states $\{h_1, \cdots, h_t\}$. Similarly, we convert an audience interaction feature into its hidden vector $g$, and the corresponding sequence into a series of hidden states $\{g_1, \cdots, g_t\}$. As each of two hidden states $g$ and $h$ at a time point is mutually influenced by the previous state of the other, we train the model by jointly optimizing $LSTM_I$ and $LSTM_A$ layers.
The initial states of CLSTM parameters are randomly initialized and tuned during training. We use $h_t$ as the latent representation for the action recognition feature and $g_t$ as that for the audience interaction feature at $t$. Then, we use a decoder $De_I$ to generate $\hat{x}_t$ from $h_t$ and another decoder $De_A$ to generate $\hat{a}_t$ from $g_t$ that are the prediction of influencer behaviours and that of audience behaviours at time $t$ respectively, as described below.
\begin{equation}
    \hat{f_t} = De_I(h_t) \quad (a)\quad 
    \hat{a_t} = De_A(g_t) \quad (b)
\end{equation}
Parameters of $De_I$ and $ De_A$ will be learnt during training. We use the Adam optimizer to learn the parameters of the CLSTM. We build CLSTM over each time slot which is set to 250 frames as \cite{DBLP:conf/wacv/GeestT18}. The input of CLSTM is series of features extracted from 64-frame segments \cite{DBLP:conf/cvpr/CarreiraZ17} with sliding step $25$ \cite{DBLP:journals/sensors/LeeCLW19}. Hence the length of sequence $s_I$ and $s_A$ is set to $9$ to cover all frames in each time slot. 
Following \cite{DBLP:conf/icde/Liu0CB20,DBLP:conf/iccv/XiaCWHS15}, we adopt a reconstruction strategy for building the objective function in the training. As we train the CLSTM by jointly optimizing $LSTM_I$ and $LSTM_A$ layers to capture the mutual influence of influencer behaviours and audience behaviours, we formulate our loss function based on reconstruction error considering the interactions between $LSTM_I$ and $LSTM_A$. 
Given a $LSTM_I$, we can use different loss functions such as JS divergence, KL divergence and $L_2$ to compute the reconstruction errors. We compare the model prediction quality using different loss functions over four datasets, INF, SPE, TED, and TWI, to choose a good one, as reported in Table \ref{tab:auc_cmp_Loss}.
\begin{table}[htb] \vspace{-1ex}
    \begin{center}
        \begin{tabular}{|l|l|l|l|l|}
            \hline
            Method &  INF &SPE  & TED  & TWI\\
            \hline
            CLSTM+$L_2$         & 76.44 &60.06  &62.90 &72.21\\
            CLSTM+$KL$       &  78.12 & 62.31 & 67.78&75.26\\
            CLSTM+$JS$        & 79.88  & 64.53 & 69.05&77.40\\
            \hline
        \end{tabular}
    \end{center}
    \vspace{-2ex}
    \caption{\small AUROC under different loss functions.}\label{tab:auc_cmp_Loss}\vspace{-2ex}
\end{table}

JSE performs best as reported in Table \ref{tab:auc_cmp_Loss}. We compute the JS loss between its input and the reconstructed output by JS divergence Error (JSE) since action recognition features are probability distributions in a normalized space. Given a $LSTM_A$, we utilize Mean Squared Error (MSE) to compute the L2 loss between its input and the reconstructed output. We fuse the reconstruction errors of $LSTM_I$ and $LSTM_A$ to form an overall loss function for the $CLSTM$ as below.
\begin{equation}
\label{eq:weight}
    l(I,A) =  \omega JSE(\hat{I}, I) + (1- \omega)MSE(\hat{A}, A)
\end{equation}
where $JSE(\hat{I}, I)$ is the reconstruction error for action recognition features, $MSE(\hat{A}, A)$ is that for audience interaction features, $\omega$ is the parameter adjusting the weight of two parts in the loss function of CLSTM. We train CLSTM with the objective of minimizing the overall information loss $l(I,A)$.
Given $n$ training data, a sequence length $q$, the dimensionality of hidden states $h_1$ in $LSTM_{I}$ and that of hidden states $h_2$ in $LSTM_A$, the dimensionality of action feature $d_1$ and that of audience interaction feature $d_2$, the time cost of training CLSTM is $\mathcal{O}(n(4q(h^2_1+h^2_2)+4q(d_1h_1+d_2h_2)))$. Moreover, following \cite{DBLP:conf/cvpr/LiZG18,DBLP:conf/cvpr/HeZRS16}, we report that the parameter number is 1,382,713, the number of floating-point operations is$1.2\times 10^7$, and the model depth is $10$.
\eat{we report Parameter Number (PN), the number of FLoating-point OPerations (FLOPs), and model depth to show the model complexity in Table \ref{tab:model_complexity}. 
\begin{table}[htb] \vspace{-2ex}
    \begin{center}
        \begin{tabular}{|l|l|l|l|}
            \hline
            Method &  PN &FLOPs  & Depth  \\
            \hline
            CLSTM         & 1,382,713 &$1.2\times 10^7$  &10 \\
            \hline
        \end{tabular}
    \end{center}
    \vspace{-2ex}
    \caption{\small Model complexity. }\label{tab:model_complexity}\vspace{-4ex}
\end{table}
}

\subsection{Anomaly Detection}
\label{sec_RENS}

Using CLSTM with a decoder layer, we can predict the behaviour of influencer and that of audience interaction at the next time point. As our final goal is to identify if the next video segment is an anomaly, we need to design a scoring function based on the output of the decoder layer. Given that the trained data are normal,  CLSTM outputs the reconstructed feature generated by the model prediction and reflects the behaviours of influencer and audience interaction in the normal status. Thus the difference between the reconstructed feature and the true incoming feature at a certain time point reflects how far the incoming feature is from the normal status. This difference is called reconstruction error, and has been proved to be effective for anomaly detection \cite{DBLP:conf/cvpr/0003CNRD16}. In this work, we extend the reconstruction error measurement for single reconstructed visual feature to a pair of reconstructed action and audience interaction features with their importance considered. Given a trained model, the reconstruction error with respect to the action feature is computed by the $JS$ divergence between the true feature $f_t$ and the reconstructed feature $\hat{f}_t$, and that for the audience interaction is computed by the $L_2$ distance between the true feature and the reconstructed one. Let m=($f_t$+$\hat{f}_t$)/2. The reconstruction errors are computed by:
\begin{equation}\label{equ:jsd}
    RE_I(t) = \frac{1}{2}\sum_{i=1}^d \hat{f}_t^i log\frac{\hat{f}_t^i}{m_t^i} + \frac{1}{2}\sum_{i=1}^d f_t^i log\frac{f_t^i} {m_t^i}
\end{equation}
\begin{equation}\label{equ:L2d}
    RE_A(t) = ||\hat{a}_t - a_{t}||_2 
\end{equation}
We fuse $RE_I$ and $RE_A$ with their importance considered, and define a weighted anomaly scoring function as follow:
\begin{equation}
\label{eq:score}
    RE_{IA}(t) = \omega RE_I(t) + (1-\omega) RE_A(t)
\end{equation}
$\omega$ is a weighting parameter to be evaluated in Section \ref{sec:exp}. Given an anomaly score threshold $\tau$, a video segment with an anomaly score bigger than $\tau$ is an anomaly.

\subsection{Dynamic Updating}
\label{sec_DU}

Practically, audience interests may change over a long stream. A speaker action may attract audience attention at the beginning, but not be attractive for audience later. Also, normal video content evolves over streams, and new one appears over streams. Accordingly, existing normal video segments used for training CLSTM cannot reflect the normal behaviours of influencers and audience interactions completely, which causes the model drift. The previously trained CLSTM may not generate accurate prediction for the incoming data. Thus, we need to dynamically maintain CLSTM for high-quality anomaly detection. We propose a smooth update approach that incrementally updates the model as required for streams by monitoring the changes of incoming data and adaptively selecting update operations. 
We take the action hidden state of $LSTM_I$, $h_t$, as the data source for detecting the model drift, since it is more robust to scene changes compared with audience interaction features. Given a set $S_{h}$ of all the hidden states for historical data and a set $S_{n}$ of all the hidden states for incoming data, the drift is measured based on the cosine similarity between the hidden vectors in these two sets.
\begin{equation}\label{eq:drift} 
sim(S_{h},S_{n})=\frac{1}{\|S_{h}\|\|S_{n}\|}\sum_{i=1}^{\|S_{h}\|}\sum_{j=1}^{\|S_{n}\|}cos(h_{i},h_{j})
\end{equation}

\eat{
\begin{algorithm}[!ht]
\DontPrintSemicolon
\KwInput{ $CLSTM$: previously trained model;\;
            \hspace*{9.5 ex}$l_s$: maximal length of set $S_{n}$;\;
            \hspace*{9.5 ex}$S_I$: incoming action interaction series;\;
            \hspace*{9.5 ex}$S_A$: incoming audience interaction series;\;
            \hspace*{9.5 ex}$S_{h}$: hidden states set of history data;\;
            \hspace*{9.5 ex}$T$: threshold to label normal segments;\;
            \hspace*{9.5 ex}$\tau$: threshold to trigger update.\;
        }
\KwOutput{$CLSTM^{t}$: updated model;$S_{h}^{t}$: updated set.}
    $n_{tmp}\gets\emptyset$\;
    $S_{n} \gets \emptyset$\;
  \For{each incoming segment $n_i$}
  {
    $h_i$ $\gets CLSTM^{t-1}(n_i)$\;
    \If{normalized audience interaction $< T$}
    {
        $n_{tmp} \gets n_{tmp}\cup n_i$\;
        $S_{n} \gets S_{n} \cup h_i$ \;
    }
    \If{$|S_{n}| == l_s$}
    {
        Compute $sim(S_{h}^{t-1},S_{n})$\;
        Update audience interaction normalization\;

        \If{$sim(S_{h}^{t-1},S_{n}) < \tau$}
        {
            $CLSTM^{t}=CLSTM^{t-1}$
        }
       \Else
       {
            Train $CLSTM^{new}$ with $n_{tmp}$\;
            $CLSTM^{t}$=merge($CLSTM^{new}$, $CLSTM^{t-1}$)\;
            $n_{tmp}\gets\emptyset$ \;
       }
        $S_{h}^{t} \gets S_{h}^{t-1} \cup S_{n}$;~
        $S_{n} \gets \emptyset$\;
    }
  }
  Return $CLSTM^{t}$ and $S_{h}^{t}$.\;
\caption{\small Dynamic Updating}
\label{alg:dynamic_update}
\end{algorithm}
}
\begin{figure}[t]\centering\small\vspace{-1mm}
	\hspace*{-0.0cm}\fbox{
		\begin{minipage}{82mm}		\textbf{input:} $CLSTM$: trained model;~$l_s$: maximal length of set $S_{n}$;\\
		 \hspace*{6.5ex} $S_I$, $S_A$: incoming action/ audience interaction series;\\
          \hspace*{6.5ex} $S_{h}$: hidden states set of history data;\\ 
            \hspace*{6.5ex} $T$,$\tau_u$: threshold to label normal segments/ trigger update.\\
		~~~~~~~\textbf{output:} $CLSTM^{t}$: updated model;            ~$S_{h}^{t}$: updated set.\\
            \hspace*{0.5ex} 1. $n_{tmp}\gets\emptyset$;~$S_{n} \gets \emptyset$; \\
			\hspace*{0.5ex} 2. \textbf{for} each incoming segment $n_i$\\
			\hspace*{0.5ex} 3. \hspace*{2.5ex} $h_i$ $\gets CLSTM^{t-1}(n_i)$\\
			\hspace*{0.5ex} 4. \hspace*{2.5ex} \textbf{if} normalized audience interaction $< T$\\
			\hspace*{0.5ex} 5. \hspace*{6.5ex} $n_{tmp} \gets n_{tmp}\cup n_i$;\quad
            $S_{n} \gets S_{n} \cup h_i$\\
			\hspace*{0.5ex} 6. \hspace*{2.5ex} \textbf{if} $|S_{n}| == l_s $\\
			\hspace*{0.5ex} 7. \hspace*{6.5ex} Compute$sim(S_{h}^{t-1},S_{n})$; UpdateAudiInteractNorm\\
            \hspace*{0.5ex} 8. \hspace*{2.5ex} \textbf{if} $sim(S_{h}^{t-1},S_{n}) > \tau_u$ \\
            \hspace*{0.5ex} 9. \hspace*{6.5ex} $CLSTM^{t}=CLSTM^{t-1}$\\
            10. \hspace*{2.5ex} \textbf{else}\\
            11. \hspace*{6.5ex} Train $CLSTM^{new}$ with $n_{tmp}$;\\
            12. \hspace*{6.5ex} $CLSTM^{t}$=merge($CLSTM^{new}$, $CLSTM^{t-1}$)\\
            13. \hspace*{6.5ex} $n_{tmp}\gets\emptyset$ ;\\
            14. \hspace*{2.5ex} $S_{h}^{t} \gets S_{h}^{t-1} \cup S_{n}$;\quad
            $S_{n} \gets \emptyset$\\
			15. Return $CLSTM^{t}$; $S_{h}^{t}$ 
			\end {minipage}
		} \vspace{-1mm}
  \caption{\small Dynamic updating}\label{alg:dynamic_update}\vspace{-2ex}
\end{figure}

Fig. \ref{alg:dynamic_update} shows the model update algorithms. Given a CLSTM trained previously, incoming action and audience interaction sequences, history hidden state set, length of set $S_{n}$, threshold to label normal segments and threshold for triggering update, our algorithm performs the dynamic update in four steps: (1) use audience interaction to label normal segments for selecting candidate states (lines 4-5); (2) use a trigger to identify the occurrence of model drift (lines 6-8); (3) update the model by merging the previous model and the new model trained by the incoming segments if the drift occurred (lines 11-13); (4) update the hidden states set of history data by appending the hidden states set of incoming data (line 14). We initialize an empty set $n_{tmp}$ to store the incoming segments to be trained in the future, and an empty set $S_n$ to store the corresponding hidden states of $n_{tmp}$ (line 1). For each incoming video segment $n_i$, we first get its hidden state $h_i$ based on previous trained $CLSTM^{t-1}$ (line 3). Then we recursively filter segments with low audience interaction (i.e $<T$), add them into $n_{tmp}$ and their hidden states into $S_n$, until the size of $|S_{n}|$ reaches the preset maximal length $l_s$ of $S_n$ (lines 4-5). We compute the similarity between  $S_{h}^{t-1}$ and $S_{n}$ using Eq.\ref{eq:drift} to trigger the drift of CLSTM (lines 6-8). If the drift is not triggered, we keep current model (line 9). Otherwise, previous $CLSTM$ is updated (lines 11-13). To obtain a new model, we first train $CLSTM^{new}$ with segments in temporary set $n_{tmp}$, and then construct the updated $CLSTM^{t}$ by merging the previous  $CLSTM^{t-1}$ and the current trained  $CLSTM{new}$ (line 12). We empty the used $n_{t p}$ (line 13) and update $S_h^{t-1}$ with $S_n$, and empty $S_n$ for next time model update (line 14). With our dynamic update algorithm, the CLSTM can be maintained incrementally, which avoids the costly model re-training operations to keep the high-quality model prediction. As discussed in Section \ref{sec:traing_strategy}, the time cost of re-training CLSTM is $\mathcal{O}(n(4q(h^2_1+h^2_2)+4q(d_1h_1+d_2h_2)))$. The time cost of CLSTM training depends on the size of training data $n$. The incremental updating allows the CLSTM to be trained over the incoming data only, which incurs lower time cost compared with the CLSTM re-training over the whole dataset.

\eat{ Given $n$ training data, a sequence length $q$, the dimensionality of hidden states $h_1$ in $LSTM_{I}$ and that of hidden states $h_2$ in $LSTM_A$, the dimensionality of action feature $d_1$ and that of audience interaction feature $d_2$, the time cost of re-training CLSTM is $\mathcal{O}(n(4q(h^2_1+h^2_2)+4q(d_1h_1+d_2h_2)))$. The time cost of CLSTM training depends on the size of training data $n$. The incremental updating allows the training to be performed over the incoming data only, leading to much lower time cost comparing with the CLSTM re-training strategy over the whole training set.
}

\section{Anomaly Detection Optimization}\label{sec:opti} 

In this work, each action recognition feature has 400 dimensions. When using JSE measure for reconstruction error calculation, the time cost for anomaly detection becomes high due to the high dimensionality of the features. We propose adaptive dimension group and adaptive optimisation to select the suitable matching strategy for each video segment.

\subsection{ADG-based Dimensionality Reduction}
We first propose a novel $Adaptive$ $Dimension$ $Group$ representation (ADG) for action recognition features. Then we propose an upper bound and show its non-trivial bound lemma. 

\subsubsection{Adaptive Dimension Group}

To reduce time cost and construct ADG for action recognition features, we transform each normalized feature vector into feature sets, each of which is described as a pair of its maximal and minimal values. To this end, we partition a dimension of feature space into variable sized segments based on the feature distributions. Each feature dimension is valued in (0, 1) space. Statistically, the smaller dimension values are distributed denser than the bigger ones. Thus we divide a dimension space into two equal parts, and recursively conduct the binary partition over the part with smaller values in each loop, until $n$ subspaces are obtained. Fig. \ref{fig:DSPar} (a) shows dimension space partition. The group $ids$ are stored in an array as shown in Fig. \ref{fig:DSPar} (b).

\begin{figure}[hthb]\vspace{-2ex}
\centering
    \subfigure[]{
    \includegraphics[width=3.5 cm]{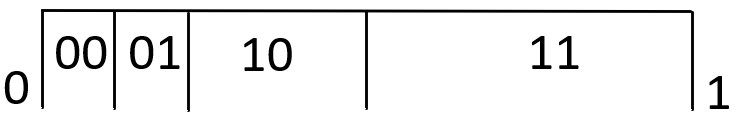}
    }
    \subfigure[]{
    \hspace{3ex}\includegraphics[width=3.6 cm]{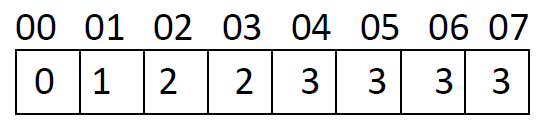}
    }\vspace{-2ex}
    \caption{\small Dimension partition vs. mapping to group $id$.}\vspace{-2ex}
    \label{fig:DSPar}
\end{figure}

Given an action recognition feature, our dimensionality reduction includes two steps. First, each feature is mapped into a set of feature groups by hash mapping using the function $h(k)=floor(k\times  2^{n-1})$. The hash key indicates the index of the array where the group $id$ to the current dimension value is obtained. Then each group is described as a pair of its minimal and maximal values, $<f_{min}^t, f_{max}^t>$. 
We conduct statistics over INF dataset by computing the minimal feature contribution (MFC) of each dimension to the final JSE calculation. As reported in Table \ref{tab:hashsubspectNum}, when we have 20 subspaces, the MFC value is close to 0. Thus, we set $n$ to 20. 

\begin{table}[htb]\vspace{-1ex} 
    \begin{center}
        \begin{tabular}{|l|c|c|c|c|c|c|}
            \hline
          $n$ & 15  & 16 & 17 & 18  & 19 & 20\\
            \hline
          MFC &0.04  & 0.02   & 0.017 & 0.012  & 0.007& 0.004\\
            \hline
        \end{tabular}
    \end{center}
    \vspace{-2ex}
    \caption{\small Filtering power of bounds.}
    \vspace{-2ex}
    \label{tab:hashsubspectNum}
\end{table}

\subsubsection{Bounding Measures}
Given an action recognition feature $f$ and its predicted feature $\hat{f}$, we define an upper bound measure of the $RE_I$ by calculating it over their representations. Let $f_t$ and $\hat{f}_t$ be two dimension groups in $f$ and $\hat{f}$, each containing $m$ corresponding dimensions. We measure their difference by:
\begin{equation}
    RE_I^{g_i}=\frac{m}{2}log\frac{max(f_{max},\hat{f}_{max})*min(f_{min},\hat{f}_{min})}{M_{min}*M_{max}}
\end{equation}

\noindent where $M_{min}$ is the minimal value of $(f+ \hat{f})/2 $. Thus, their overall representation difference is:
    $RE_I^G=\sum RE_I^{g_i}$.

\begin{theorem}\label{theo:Upmax}
$RE_I^G$ is an upper bound of $\widetilde{RE_I}$ between the action recognition feature $f$ and its estimated feature $\hat{f}$.
\end{theorem}

\begin{proof}
For a dimension $i$ in a group $g_t$, if $f_i\geq\hat{f}_i$, we have:
\begin{equation}\label{equ:max}
log(max(f_{max},\hat{f}_{max})/M_{min})\geq log(max(f_{i},\hat{f}_{i})/M)\geq 0   
\end{equation}
\begin{equation}\label{equ:min}
0\geq log(max(f_{min},\hat{f}_{min})/M_{max})\geq log(max(f_{i},\hat{f}_{i})/M)   
\end{equation}
Combine Eq.\ref{equ:max}-\ref{equ:min}, we have 
\begin{equation}\label{equ:maxmin}
\begin{array}{l}
log\frac{max(f_{max},\hat{f}_{max})}{M_{min}}+log\frac{max(f_{min},\hat{f}_{min})}{M_{max}}\geq \\ log\frac{max(f_{i},\hat{f}_{i})}{M}+log\frac{max(f_{i},\hat{f}_{i})}{M}  
\end{array}
\end{equation}
Similarly when $f_i<\hat{f}_i$, Eq.\ref{equ:maxmin} holds. Thus for the whole group with $m$ dimensions, we have:
\begin{equation}\label{equ:maxmin_all}
\begin{array}{l}
\frac{m}{2}log\frac{max(f_{max},\hat{f}_{max})}{M_{min}}+\frac{m}{2}log\frac{max(f_{min},\hat{f}_{min})}{M_{max}}\geq \\ \frac{m}{2}log\frac{max(f_{i},\hat{f}_{i})}{M}+\frac{m}{2}log\frac{max(f_{i},\hat{f}_{i})}{M}  
\end{array}
\end{equation}
Thus we have 
    $RE_I^G\geq RE_I$. We can conclude that 
$RE_I^G$ is an upper bound of $RE_I$.
\end{proof}

As in \cite{61115}, we also use the $L_1$-based bounds. Upper bound $JS_{max}$ and lower bound $JS_{min}$ of JS divergence in terms of $L_1$ norm can be described by:
$JS_{max}(P,Q) \leq 0.5*{\|P-Q\|_1}$ and $JS_{min}(P,Q) \geq 0.125*{\|P-Q\|^2_1}$. Our ADG-based upper bound and $L_1$-based bounds can be used for filtering jointly to optimize the detection time cost.

\subsection{Adaptive Optimisation}
Once we obtain the lower bound and upper bound measures, we can use them to filter out the potential normal or abnormal video segments without the final $RE_I$ calculation in anomaly identification. Each bound has different filtering power, leading to different performance improvements. We conduct statistics over a sample dataset of 8921 video segments, and obtain the filtered sample numbers of three bound measures, $JS_{max}$=2273, $JS_{min}$=1490 and $RE_I^G$=3870.\eat{The statistics on the filtering power of three bound measures, $JS_{max}$, $JS_{min}$ and $RE_I^G$ on a sample dataset including 8921 video segments is detailed in Table~\ref{tab:filter_power}.}
Clearly, the ADG-based bound and the combination of $L_1$-based bounds have filtered similar number of samples, which indicates the unignorable filtering power of each bound. Thus, we need to design a strategy to integrate all the bounds in a smart way to jointly improve the performance of anomaly detection. 

\begin{figure}[t]\centering\small
	\hspace*{-0.0cm}\fbox{
		\begin{minipage}{82mm}		\textbf{input:} $CLSTM$: trained model; ~\quad$T_1, T_2$: ADOS thresholds;\\
          \hspace*{6.5ex} $S_I$, $S_A$: incoming action/ audience interaction series;\\
          \hspace*{6.5ex} $T_n, T_a$: normal/anomaly thresholds;\\
		~~~~~~~\textbf{output:} $S_n$: a set of abnormal video segments.\\
            \hspace*{0.5ex} 1. $S_{n} \gets \emptyset$; \\
			\hspace*{0.5ex} 2. \textbf{for} each incoming segment $s_i$ with feature $f_i$\\
			\hspace*{0.5ex} 3. \hspace*{2.5ex} $\hat{f}_i$ $\gets CLSTM(f_i)$;\quad
        $tF=tFunc(f,\hat{f})$ \\
			\hspace*{0.5ex} 4. \hspace*{2.5ex} \textbf{if} $T_1\leq tF \leq T_2 $ \quad /* not skip $L_1$-based filtering */\\
			\hspace*{0.5ex} 5. \hspace*{6.5ex} Compute $JS_{max}(f_i,\hat{f}_i)$, $JS_{min}(f_i,\hat{f}_i)$; \\
			\hspace*{0.5ex} 6. \hspace*{2.5ex} \textbf{if} $JS_{max}<T_n$\\
			\hspace*{0.5ex} 7. \hspace*{6.5ex} Continue; \quad /*go to process next segment */\\           
			\hspace*{0.5ex} 8. \hspace*{2.5ex} \textbf{else if}  $JS_{min}>T_a$ \quad /*is an anomaly */ \\
            \hspace*{0.5ex} 9. \hspace*{10.5ex} $S_{n} \gets S_{n} \cup s_i$; Continue;\\     
            10. \hspace*{2.5ex} $RE_{max}\gets RE_I^G (ADG(f),ADG(\hat{f}))$;\\
            11. \hspace*{2.5ex} \textbf{if} $RE_{max} > T_n$ \\
            12. \hspace*{6.5ex} $RE\gets RE_I$;\\
            13. \hspace*{6.5ex} \textbf{if} $RE>T_a$\\
		  14. \hspace*{10.5ex} $S_{n} \gets S_{n} \cup s_i$\\
            15. Return $S_{n}$
			\end {minipage}
		} \vspace{-1mm}
  \caption{\small Anomaly detection with ADOS filtering.}
\label{alg:ADOS-AOVLIS}\vspace{-4mm}
\end{figure}

A naive method to integrate these three bounds is to simply apply $JS_{min}$, $JS_{max}$ and $RE_I^G$ to the filtering process one by one. However, there are some cases under which the $L_1$-based bounds do not work. If we can predict these cases, we can save time for the calculation of $L_1$ distance of features in the filtering, leading to better overall performance. To achieve this, we design an ADaptive Optimisation Strategy (ADOS) that enables the system to select the best bounds to optimize the final anomaly identification. In ADOS, it is important to define a trigger to automatically decide a suitable bound to be selected for filtering. 
We define the trigger function based on the characteristics of video segment features. Intuitively, for the 400-dimensional action recognition feature vectors, the sum of all dimension values equals 1, and only 1-3 dimension values are bigger than 0.1. Given an action recognition feature $f$ and its estimated one $\hat{f}$, if the dominant dimension value of $f$ is close to the value of $\hat{f}$ on the corresponding dimension, it indicates $f$ and $\hat{f}$ have a similar distribution in terms of their dominant dimension values and positions. Otherwise, they do not have similar distributions. Thus, we simply define a trigger function based on the difference between the features on the dominant dimension $i$ of the $f$, which is computed by:
\begin{equation}
    tFunc(f,\hat{f})=\|f_{i}-\hat{f}_i\|
\end{equation}

When tFunc value $tF$ is smaller than a threshold $T_1$, $JS_{max}$ will be very small, which is not suitable for filtering out the normal video segments. If $tF$ is bigger than a threshold $T_2$, $JS_{min}$ will be very big, which is unsuitable for filtering out the anomaly. Thus, only if $tF$ is a value within $(T_1, T_2)$, we will apply $L_1$-based bound for filtering. Fig.\ref{alg:ADOS-AOVLIS} details the whole anomaly detection process with ADOS filtering. Given a series of video segments described as a pair of series $S_I$ and $S_A$, our algorithm performs three main steps. First, ADOS is conducted to trigger the $L_1$-based filtering (lines 2-10). Then, for the segments that could not be filtered by $L_1$-based bounds, we conduct ADG-based filtering (lines 11-12). Following that, for the unfiltered video segments, we compute the RE distance between the original feature and the estimated one, and output the segment if it is an anomaly (lines 13-14). The set of anomaly $S_n$ is finally returned (line 15).

\section{Effectiveness Evaluation}\label{sec:exp}
We evaluate the effectiveness and efficiency of AOVLIS for anomaly detection over social video live streaming.

\subsection{Experimental Setup}
\label{sec:exp:setup}
We conduct experiments over 212 hours social videos $D_w$ from two sources: 
(1) A set of 84 hours Bilibili videos, containing 42 hours original live streams downloaded from Bilibili \cite{bilibili} and 42 hours transformed copies generated by horizontally flipping these original streams; (2) A 128 hours video set including 64 hours original videos from Twitch \cite{twitch} and 64 hours transformed copies with the horizontally flipping.
The Bilibili video set includes three types of videos that constitute three datasets: (1) INF consisting of 31 hours of influencer videos; (2) SPE including 21 hours of speech videos; (3) TED including 32 hours of TED videos. 
Videos downloaded from Twitch constitute the fourth dataset: (4) TWI including 128 hours of Twitch videos. Here, an influencer video is a type of sponsored content created by an influencer to promote a brand. A speech video contains a formal presentation/talk held by the speakers. A TED video is created by expert speakers on education, business, science, tech, and creativity. 
A Twitch video is mainly about an influencer playing games.
We preprocess the original videos by resizing the frame rate to 25 fps and horizontally flipping frames to augment data. We select videos 40 hours from four datasets (6h for $S_I$, 4h for $S_S$, 6h for $S_T$, 24h for $S_W$) for testing. The ground truth is obtained by the user study involving 5 assessors majoring in computer science based on anomaly definition and inter-subject agreement on the relevance judgments.

We compare our CLSTM with three anomaly detection approaches LTR \cite{DBLP:conf/cvpr/0003CNRD16}, VEC \cite{DBLP:conf/mm/YuWCZXYK20}, and RTFM \cite{DBLP:conf/iccv/TianPCSVC21}, and two LSTM-based alternatives LSTM and CLSTM-S. LTR is an autoencoder-based method that reconstructs regular motion with low error. VEC trains DNNs to infer the erased patches from incomplete video events series. RTFM is a state-of-the-art video anomaly detection that uses a temporal feature magnitude to learn more discriminative features.
LSTM in this evaluation applies the LSTM over the action features only. CLSTM-S captures both action feature and audience interaction feature but only considers the single-way impact of influence behaviour on audience behaviour.

The effectiveness of our model is evaluated in terms of two popular metrics used for measuring the anomaly detection quality: receiver operating characteristic (ROC) curve and corresponding area under ROC curve (AUROC). ROC is a probability curve which is produced by measuring TPR (true positive rate) of the anomaly detection at corresponding FPR (false positive rate) points from 0 to 1. AUROC represents the degree or measure of separability, showing the discriminatory power of the model. Following \cite{DBLP:conf/cvpr/0003CNRD16}, the evaluation on the epoch parameter training is measured by the reconstruction error, $R_e$, to decide the model divergence.

We evaluate the efficiency of AOVLIS in terms of the average time cost of each video segment for anomaly detection that includes the feature extraction and reconstruction computation over the whole stream. We define a new metric, filtering power ($fp$),  to evaluate the filtering capability of bound measures. $fp$ is computed by:
$fp=\frac{the~number~of~filtered~segments}{the~ number~ of ~total~ segments}$. 
CLSTM is implemented using the Pytorch. With a learning rate of 0.001, the Adam optimizer is used considering its computing efficiency and low memory cost. By varying threshold $\tau$ from $(0,1)$, we test the effectiveness at different $\tau$ values and achieve the optimal $\tau$ values $0.182, 0.097, 0.052, 0.058$ for the INF, SPE, TED, and TWI, respectively. $T_a$ is calculated using Eq.\ref{eq:score}. 
We obtain $T_n=0.7*T_a$ with optimal model performance. For the dynamic update, we search the optimal maximal set size $l_s$ by setting it as $[100, 200, 300, 400]$ and the optimal triggering threshold $\tau_u$ by setting it to a value in $(0,1)$ for effectiveness test, and achieve the optimal $l_s = 300$ and $\tau_u=0.4$. The threshold $T$ for filtering anomalies is the average normalized audience interactions of the previous time slot.
The tests are conducted on an Intel i5, 2.30GHz processor machine with a 4 GB NVIDIA GTX 1050Ti graphics card.
The source code and data have been made available \footnote{\scriptsize{\url{https://1drv.ms/f/s!AiVqBKWYV6J8hRgAOSS3djkb6AEH?e=Q7TpdW}}}.

\subsection{Effectiveness Evaluation}
We first evaluate the effects of parameters: $P$ and $\omega$, during training, then compare our solution with the SOTA methods, followed by the test on the effect of dynamic update.
\eat{followed by the test on the audience interaction prediction and the effect of dynamic update.}

\eat{
\begin{figure}[htb]
\centering
    \subfigure[Influencer]{
    \includegraphics[width=2.6 cm]{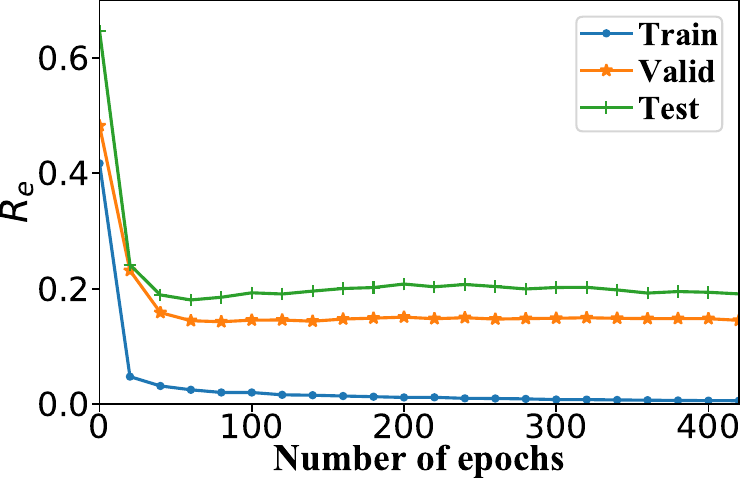}
    }
    \subfigure[Speech]{
    \includegraphics[width=2.6 cm]{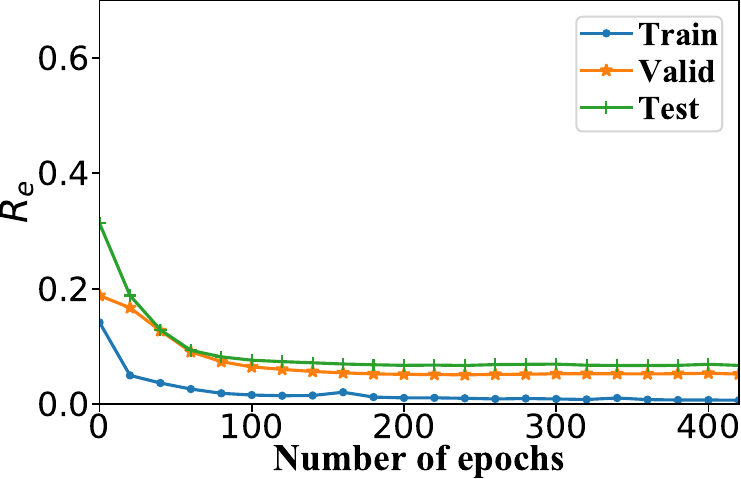}
    }
    \subfigure[TED]{
    \includegraphics[width=2.6 cm]{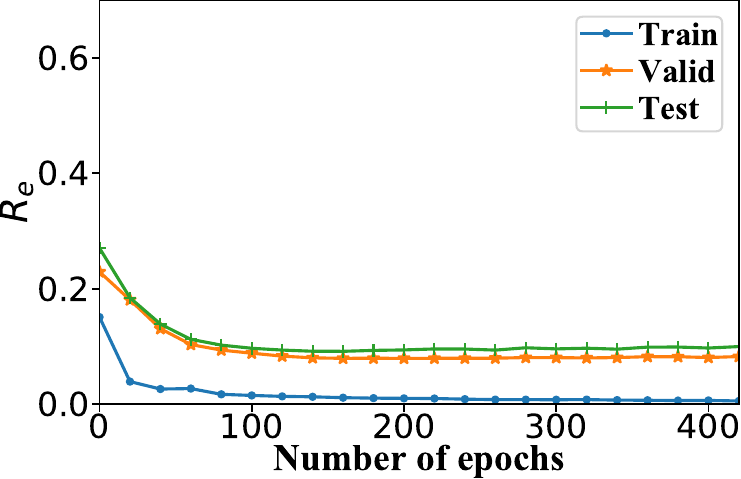}
    }\vspace{-2ex}
    \caption{\small Effect of epoch.}\vspace{-2ex}
    \label{fig:parameters_epoch}
\end{figure}
}

\subsubsection{Effect of Epoch}

We evaluate the effect of $P$ on the reconstruction error $R_e$ during CLSTM training over four video datasets by varying the $P$ from 50 to 1000. We divide the normal segments into training set ($75\%$) and validation set ($25\%$), and construct a test set from abnormal segments. Fig. \ref{fig:parameters_epoch} shows the $R_e$ value changes of training set, validation set and test set over different epoch numbers. Clearly, the $R_e$ of normal segments is remarkably different from that of anomalies, thus it is reasonable to use the learned patterns of normal data as the detector for finding anomalies. Meanwhile, for each of the datasets, the $R_e$ of training set decreases and reaches close to 0. For the validation set, the $R_e$ slightly climbs after $EPOCH=300$ due to the over-fitting in learning. \eat{This is caused by the over-fitting in learning, where the model cares too much on minimizing the training loss and performs poorly over unseen data.} Thus, we set $EPOCH=1000$ as the maximum epoch for training the model to well trade-off the loss of training set and the over-fitting of validation set, save the model every 50 epochs and test on valid set. The model getting best performance will be used as the final model for test phase.

\eat{
 \begin{figure}[htb]
    \centering
    \subfigure[ Effect of $\omega$]{
    \includegraphics[width=2.6 cm]{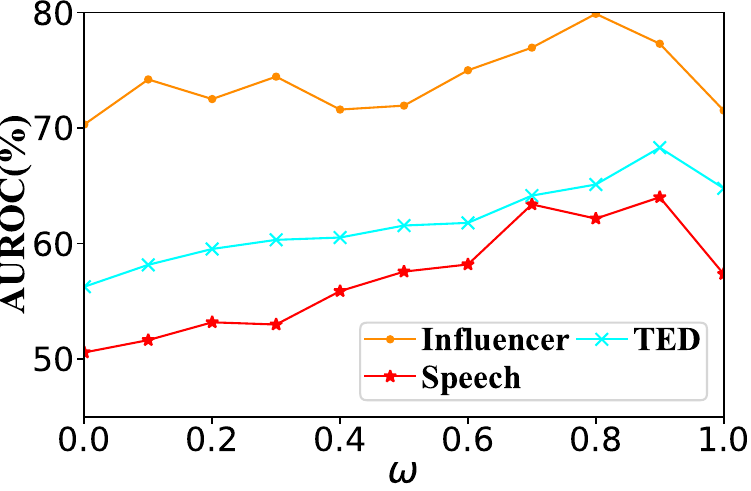}
    \label{fig:weightAUC}
    }
    \subfigure[ AUROC comp.]{
    \includegraphics[width=2.6 cm]{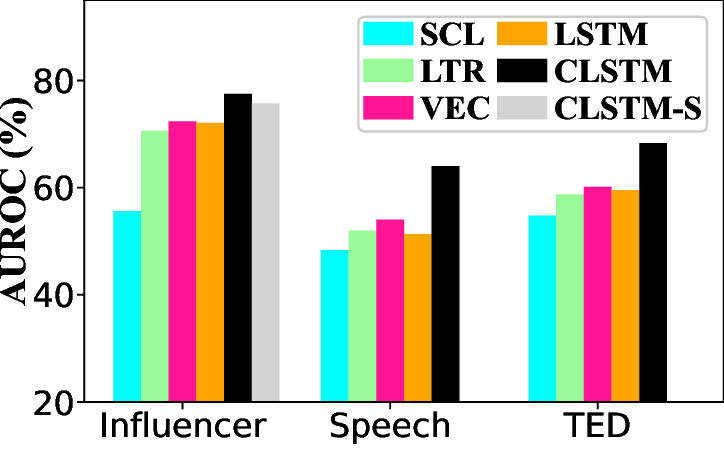}
    \label{fig:auc_cmp}
    }
    \subfigure[ Model updates]{
    \includegraphics[width=2.6 cm]{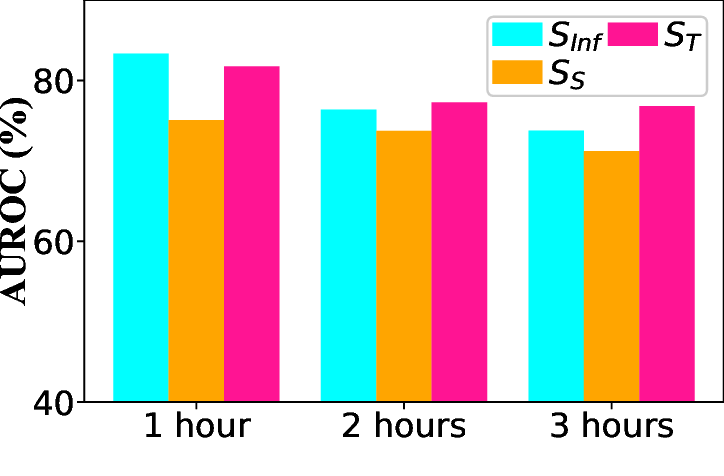}
    \label{fig:EffeModelupdate}
    }\vspace{-2ex}
    \caption{\small Effectiveness under AUROC.}\vspace{-1ex}
\end{figure}
}
\begin{figure}[hthb]\vspace{-2ex}
\centerline{
    \includegraphics[width=0.48\textwidth]
    {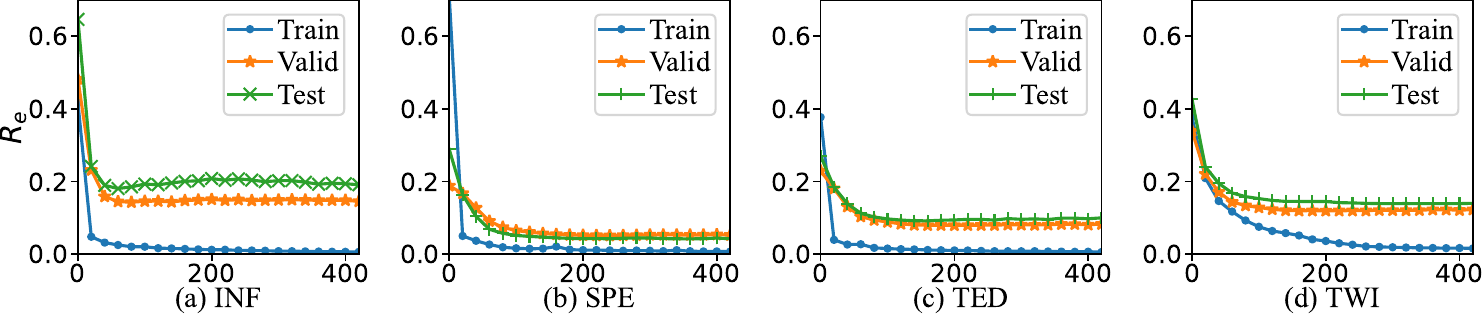}
}\vspace{-0.5ex}
\caption{\small{Effect of epoch.}}\vspace{-2ex}
\label{fig:parameters_epoch}
\end{figure}

\subsubsection{Effect of Audience Interaction Weight $\omega$}
We evaluate the effect of audience interaction weight $\omega$ in our objective function in CLSTM training. In this test, the $\omega$ is varied from 0 to 1, and the optimal epoch number is applied. Fig. \ref{fig:under_auroc} (a) reports the corresponding AUROC values with respect to different $\omega$. Clearly, CLSTM achieves the best AUROC when $\omega=0.8$ for INF, and $\omega=0.9$ for SPE, TED, and TWI. 

\subsubsection{Effectiveness Comparison}

\begin{figure}[hthb]\vspace{-2ex}
\centerline{
            \includegraphics[width=0.28\textwidth]{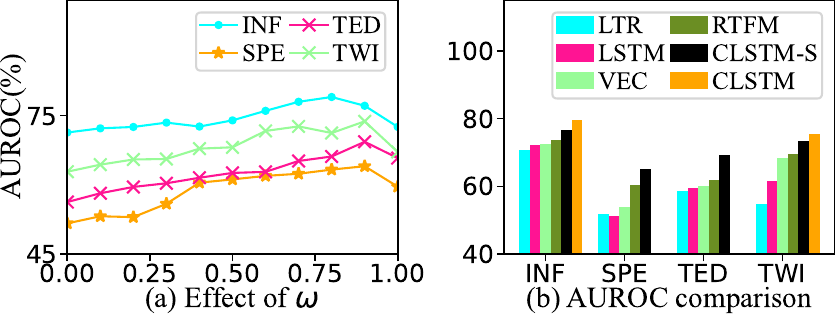}
        }
\vspace{-0.5ex}
\caption{\small {Effectiveness under AUROC.}}\vspace{-2ex}
\label{fig:under_auroc}
\end{figure}

\eat{
\begin{figure}[htb]
\centering
    \subfigure[Influencer]{
    \includegraphics[width=2.6 cm]{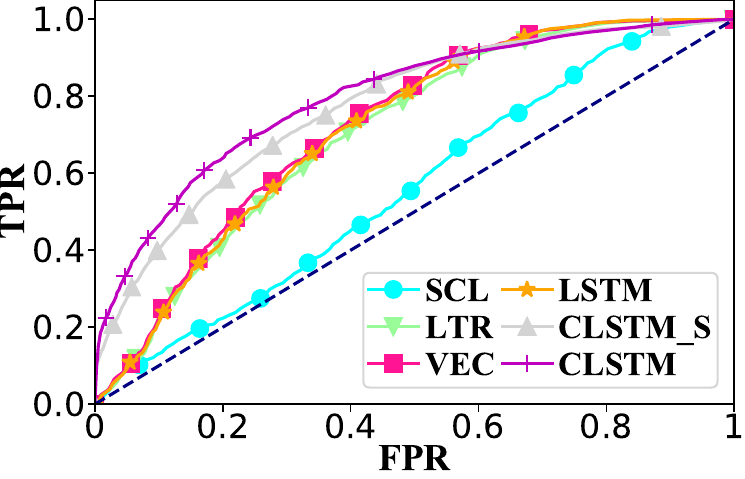}
    }
    \subfigure[Speech]{
    \includegraphics[width=2.6 cm]{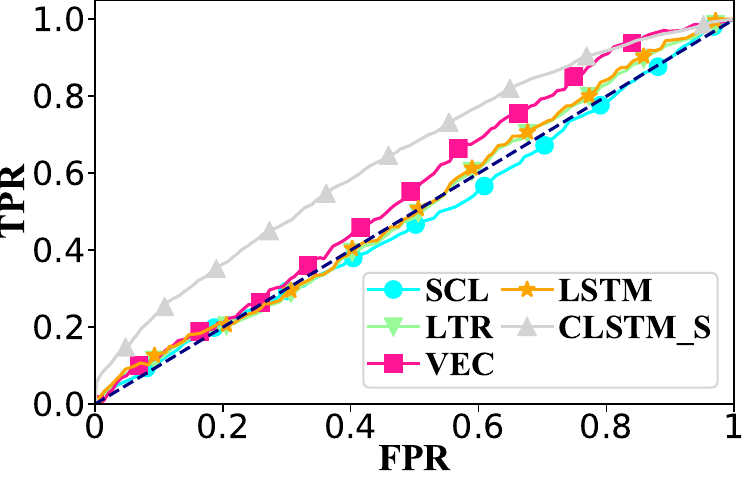}
    }
   \subfigure[TED]{
    \includegraphics[width=2.6 cm]{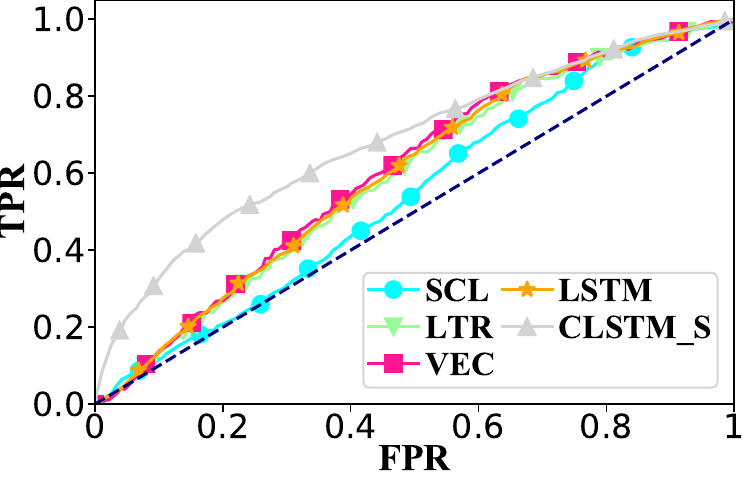}
    }
    \vspace{-2ex}\caption{\small ROC curves comparison on datasets.}
    \label{fig:roc}\vspace{-2ex}
\end{figure}
}

We compare the effectiveness of six methods, CLSTM,
LSTM, CLSTM-S, LTR\cite{DBLP:conf/cvpr/0003CNRD16}, VEC\cite{DBLP:conf/mm/YuWCZXYK20}, and RTFM \cite{DBLP:conf/iccv/TianPCSVC21}, by performing anomaly detection over
all datasets. For all the methods, we apply their optimal parameter settings. Fig. \ref{fig:roc} (a)-(d) show the ROC curve over four sets. Clearly, CLSTM-based detection obtains the best true positive rate for each false positive rate level, followed by the CLSTM-S-based method for INF and TWI data. The high effectiveness of CLSTM is obtained by its strong capability in capturing the spatial and temporal information of both influencers and audience, and their mutual influence. CLSTM-S-based method can capture the influence of speakers to audience in addition to the spatial-temporal information of both sides, thus achieving better effectiveness than other competitors. For SPE and TED, CLSTM performs the same as CLSTM-S since the characteristics of these videos that the comments from audience can not be received by speakers. VEC and RTFM are better than LSTM-based detection because they predict the representation of a video segment based on both its previous and next segments, capturing the temporal information in a bidirectional way.
Fig. \ref{fig:under_auroc} (b) reports the AUROC of different methods over four datasets. Our CLSTM-based detection achieves the best effectiveness in terms of AUROC, which further proves the high effectiveness of CLSTM. 
\begin{figure}[hthb]\vspace{-2ex}
\centerline{
    \includegraphics[width=0.48\textwidth]{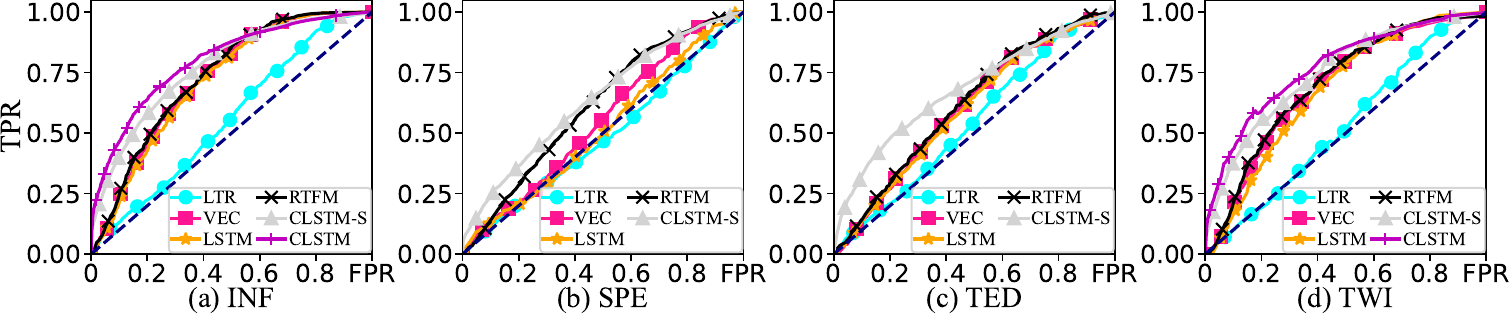}
        }
\vspace{-0.5ex}
\caption{\small{ROC curves comparison on datasets.}}\vspace{-2ex}
\label{fig:roc}
\end{figure}

\subsubsection{Effect of Dynamic Updating}

We test the effect of incremental dynamic updates and compare it with the re-training strategy on the anomaly detection effectiveness. For each dataset, we take the original trained model over the training set and incrementally update the model by varying the update frequency from every 1 to 3 hours. Likewise, we retrain the CLSTM at the same update frequencies. We detect the anomalies over streams in test set using the CLSTM with each update frequency. Table \ref{tab:effect_update} shows the AUROC of CLSTM with two dynamic update strategies. Clearly, when the update frequency is small, our incremental model obtains the best performance. This is because more frequent updates can fit the model to reflect the data update on time, thus better avoiding the model drift problem in streaming. Moreover, the incremental updating outperforms the re-training strategy, since it treats the old training and incoming training data separately. Accordingly, the part of CLSTM over the new data well reflects the recent activities of influencer and audience, and can better affect the prediction results. On the other hand, re-training strategy mixes all the sequence data and treats them equally in the training, which cannot differentiate the importance of old data and new data, leading to low AUROC.

\begin{table}[tb] 
\setlength\tabcolsep{2.pt}
    \begin{center}
        \begin{tabular}{|l|cccc|cccc|}
        \hline
           &\multicolumn{4}{|c|}{Incremental Update}& \multicolumn{4}{|c|}{Re-training}\\
         Freq.  & INF& SPE& TED&TWI &INF& SPE& TED&TWI \\
            \hline
                       1h& {83.33}& {75.06}& {81.75}& {79.42} &76.21& 70.33& 73.11&73.56 \\
            \hline
                       2h& 76.37& 73.73& 77.27&74.11 &73.79& 69.70& 71.63&72.11 \\
            \hline
                       3h& 73.77& 71.20& 76.80&72.13 &72.39& 67.66& 71.30& 70.93 \\
            \hline
          
        \end{tabular}
    \end{center}\vspace{-2ex}
    \caption{\small {Effect of incremental model updates (AUROC($\%$)).}}\vspace{-5ex}
    \label{tab:effect_update}
\end{table}

\subsection{Efficiency Evaluation}

We evaluate the efficiency of our anomaly detection in terms of overall time cost, and the cost of dynamic model update.

\subsubsection{Comparing the Filtering Power} We evaluate the filtering power of different bounds in video anomaly detection. Fig. \ref{fig:EffiComp} (a) shows the filtering powers of three bounds $JS_{max}$, $JS_{min}$ and $RE_I^G$, and their combinations, $JS_{max}$+$JS_{min}$, $JS_{max}$+$JS_{min}$ +$RE_I^G$ and ADOS. As we can see, $RE_i^G$ has similar filtering power as $JS_{max}$+$JS_{min}$, which indicates $RE_i^G$ is a tighter upper bound. In addition, the combination of all the bounds achieves the strongest filtering power, because the $L_1$-based filtering, $JS_{max}$ and $JS_{min}$, and $RE_I^G$ filter the video segments from different aspects, thus suitable to different video segments. Thus, it is necessary to optimize the efficiency by adaptively integrating three bound measures.

\begin{figure}[hthb]\vspace{-2ex}
\centerline{
            \includegraphics[width=0.45\textwidth]{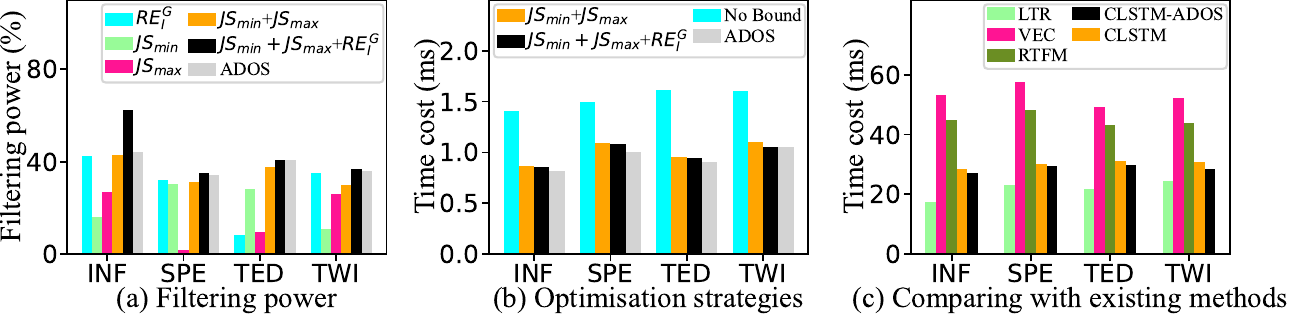}
        }
\vspace{-0.5ex}
\caption{\small Filtering power and efficiency evaluation.}
\vspace{-2ex}
\label{fig:EffiComp}
\end{figure}

\subsubsection{Effect of $T_1$ and $T_2$} We first test the effect of $T_1$ by varying $T_1$ from 1.1 to 2 over four datasets, and report the time cost of anomaly identification at each $T_1$ value. Fig. \ref{fig:FP_T} (a) reports the time cost of the detection at different $T_1$ over four datasets. As we can see, for each dataset, the time cost decreases first with the increase of $T_1$, reaches the lowest cost when $T_1$ is 1.6 for INF and TWI, 1.8 for TED and SPE, and increases again with the further increasing of $T_1$ values. This is because the number of anomalies could not be filtered when $T_1$ is small, while they are not skipped, leading to a big number of redundant $L_1$ distance calculations. There is a fluctuation of the trend due to the irregular distributions of dimension values in different subspaces. Thus we select the default $T_1$ to 1.6 for INF and TWI data, and 1.8 for TED and SPE data, respectively.

\eat{
\begin{figure}[htb]
\centering
   
    \subfigure[Effect of $T_1$]{
    \includegraphics[width=2.6 cm]{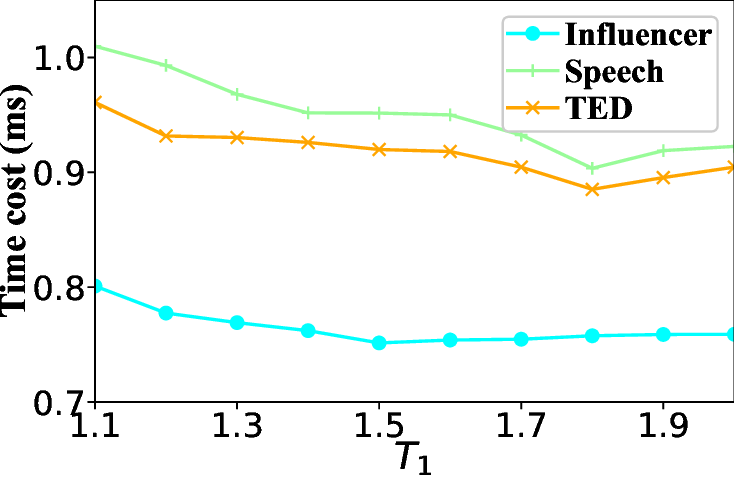}
    \label{fig:T1}
    }
    \subfigure[Effect of $T_2$]{
    \includegraphics[width=2.6 cm]{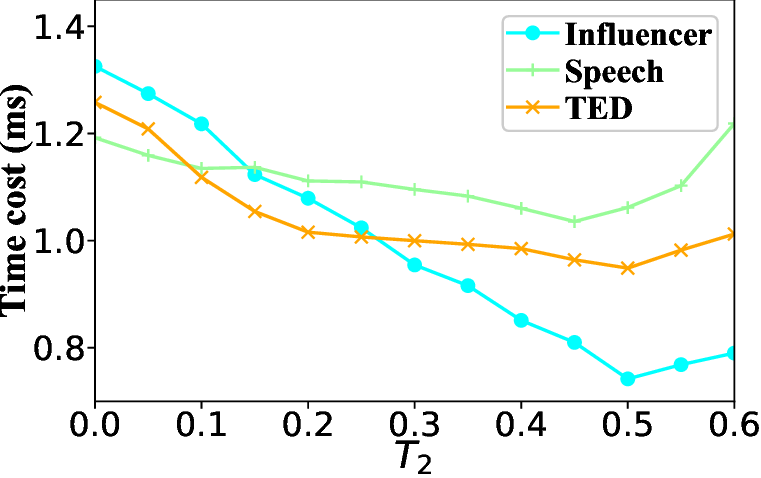}
    \label{fig:T2}
    }
     \subfigure[Effect of $N_{sg}$]{
    \includegraphics[width=2.6 cm]{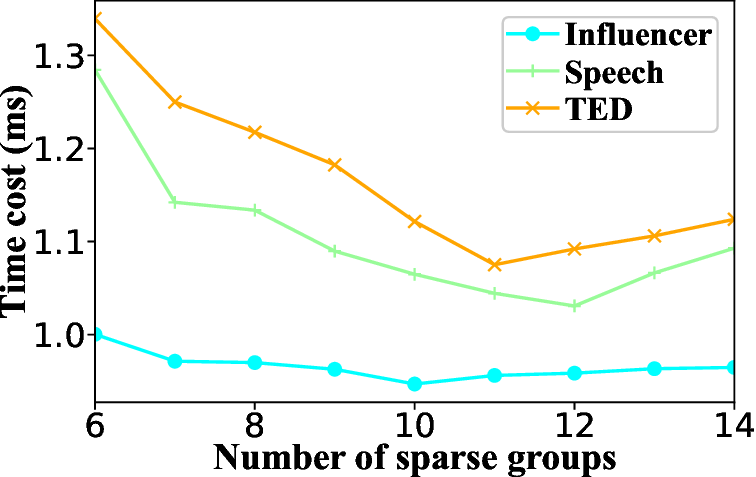}
    \label{fig:incrementalEffe}
    }\vspace{-2ex}
    \caption{\small Effect of thresholds.}\vspace{-2ex}    \label{fig:FP_T}    
\end{figure}
}

\begin{figure}[hthb]\vspace{-2ex}
\centerline{
            \includegraphics[width=0.45\textwidth]{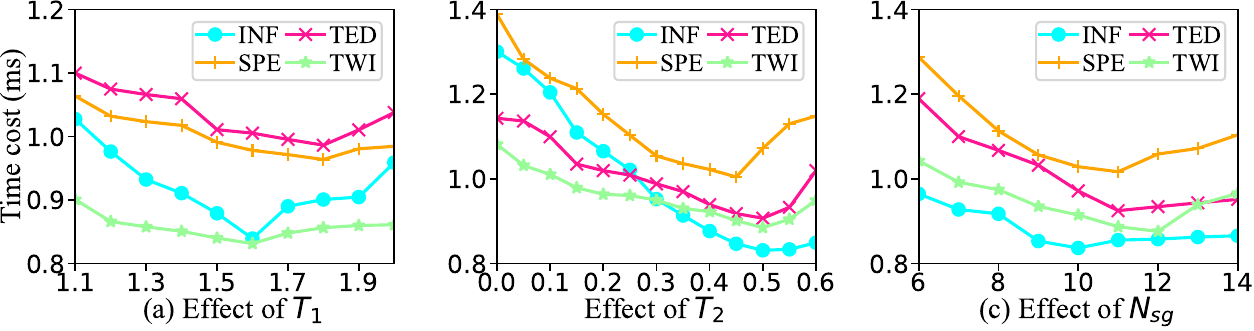}
        }
\vspace{-0.5ex}
\caption{\small Effect of thresholds.}\vspace{-2ex}
\label{fig:FP_T}
\end{figure}

We also test the time cost of anomaly identification by varying $T_2$ from 0 to 0.6, and report the time cost of anomaly detection at each $T_2$. Fig. \ref{fig:FP_T} (b) shows the time cost change of the detection over four datasets. Clearly, the best efficiency is achieved when $T_2$ is 0.5, 0.5, 0.45, and 0.5 for INF, TED, SPE, and TWI data respectively. When $T_2$ is small, a big number of normal segments cannot be filtered out with an extra $L_1$-based upper bound. Thus we select 0.5, 0.5, 0.45, and 0.5 as the default values for INF, TED, SPE, and TWI respectively.

\subsubsection{Effect of Incremental Computation in $RE_I$} 
We compute the $RE_I^G$ over the dimension groups. As the sparse groups produce loose bounds for the corresponding parts in $RE_I$ approximation, we compute the feature difference in the sparse groups in the original space whose results can be used incrementally in the final $RE_I$ distance calculation. Here, we evaluate the optimal number of sparse groups $N_{sg}$ by testing the time cost of anomaly detection with $N_{sg}$ values set to 0 to 14. Fig. \ref{fig:FP_T} (c) shows the time cost change over each dataset. Obviously, the cost trend drops with fluctuation to the best performance at $N_{sg}$=10 for INF, $N_{sg}$=11 for TED and SPE, and $N_{sg}$=12 for TWI data, and increases with fluctuation after the best $N_{sg}$ points. Thus we select the default $N_{sg}$ values, 10 for INF, 11 for TED and SPE, and 12 for TWI.


\subsubsection{Comparing Different Optimisation Strategies}
We compare the effect of different optimisation strategies by testing the time cost of anomaly detection using two optimisation strategies including the simple combination of three measures $JS_{max}$+$JS_{min}$+$RE_I^G$, the adaptive optimation strategy ADOS, and the detection without bound as a reference to evaluate how big the efficiency improvement is for each strategy. Fig. \ref{fig:EffiComp} (b) shows the test results of three approaches. Clearly, our proposed ADOS achieves higher efficiency improvement than the $JS_{max}$+$JS_{min}$+$RE_I^G$. This is because our ADOS adaptively selects the bounds for filtering, which avoids the unnecessary $L_1$-based calculation leading to invalid filtering operations. This has proved the priority of our ADOS strategy. 
 
\subsubsection{Efficiency Comparison}

We compare our CLSTM-based anomaly detection with state-of-the-art techniques in terms of the overall time cost. We test the time cost of the detection over four datasets. Fig. \ref{fig:EffiComp} (c) reports our CLSTM-based methods with the existing competitors, LTR, VEC, and RTFM in terms of the time cost for anomaly detection. Clearly, CLSTM is much faster than VEC and RTFM. Compared with LTR, CLSTM achieves comparable efficiency and much higher effectiveness. Furthermore, CLSTM-ADOS achieves the best efficiency performance because it adopts upper/lower bound filtering with an adaptive bound selection strategy which effectively reduces the distance calculation in the original visual feature space. This has proved the efficacy of CLSTM. 

\subsubsection{Cost of Dynamic Update}

To evaluate the efficiency of dynamic update, we take the previously trained model on the training set of each dataset, and use the test dataset as incoming videos to trigger the update every one hour.
The time costs for the model update over $S_{I}$,  $S_S$, $S_T$, and $S_{W}$ datasets are 174.02s, 130.32s, 143.65s, 183.42s respectively, while the re-training process requires 5.2h, 2.4h, 6.0h, 20.5h respectively to update the model over these four datasets on each iteration. Compared with re-training, the incremental strategy achieves up to 403 times improvement in terms of model update time cost. As we adopt an incremental strategy, our model is updated over the incoming data only when the trigger finds the model drift, leading to high update efficiency. However, re-training strategy requires the whole dataset to be trained in model updates, leading to high time cost for model maintenance. Thus the model update cost is well controlled with incremental update strategy for stream processing.

\subsection{Case Study of the Live Social Video Anomaly Detection}
\label{seccasestudy}
We examine the video segments from INF dataset and the detection results of our CLSTM and SOTA approaches to determine the effectiveness of different approaches. We select a test video stream randomly from INF dataset and divide it into segments. The sequential segments are fed into the learned model to generate the predicted features of the next-time segment. For each segment, the anomaly score is computed, and the anomalies are detected. 
We report the anomaly scores generated using different detection approaches and show if a segment is an anomaly based on both detection results using different methods and the ground truth labels for these samples. 
Fig. \ref{fig:VideoSegSamples} shows 15 video segment samples. The detection results and labels are reported in Table \ref{tab:detectRst}. Here, $L_p$ and $L_g$ mean the predicted and ground truth labels respectively. A video segment labeled with 1 is an anomaly, while the one labeled with 0 is a normal one. The underlined bold results generated by different methods indicate false detections. Clearly, our CLSTM and CLSTM-S produce one false detection only, while the number of wrong detections generated by existing techniques, LTR, VEC, LSTM and RTFM, are 5, 3, 4 and 3 respectively. Thus, our CLSTM achieves best detection results that are consistent to the ground truth labels. This has further proved that CLSTM outperforms the SOTA techniques. \eat{An underlined bold result indicates a false detection. As shown in Table \ref{tab:detectRst}, CLSTM and CLSTM-S perform better than existing models, with only one false detection generated. This is because CLSTM models explore both the influencer action feature and the audience interaction feature to detect anomalies.
RTFM and VEC consider temporal visual information and achieve the second-best performance by only generating three false detections. LSTM performs slightly worse since it also only considers the temporal visual information. LTR performs worst, generating 5 false detections, since this model relies on visual information and is hard to capture long-term information. }

\begin{figure}[tb]
\centering

    \subfigure[Sid: 1]{
    \includegraphics[width=1.4 cm]{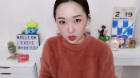}
    \label{fig:s1}
    }
    \subfigure[Sid: 1]{
    \includegraphics[width=1.4 cm]{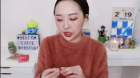}
    \label{fig:s2}
    }
    \subfigure[Sid: 3]{
    \includegraphics[width=1.4 cm]{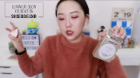}
    \label{fig:s3}
    }
    \subfigure[Sid: 4]{
    \includegraphics[width=1.4 cm]{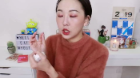}
    \label{fig:s4}
    }
    \subfigure[Sid: 5]{
    \includegraphics[width=1.4 cm]{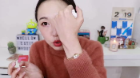}
    \label{fig:s5}
    }
      \subfigure[Sid: 6]{
    \includegraphics[width=1.4 cm]{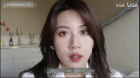}
    \label{fig:s6}
    }
    \subfigure[Sid: 7]{
    \includegraphics[width=1.4 cm]{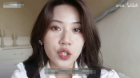}
    \label{fig:s7}
    }
    \subfigure[Sid: 8]{
    \includegraphics[width=1.4 cm]{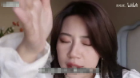}
    \label{fig:s8}
    }
    \subfigure[Sid: 9]{
    \includegraphics[width=1.4 cm]{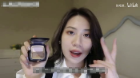}
    \label{fig:s9}
    }
    \subfigure[Sid: 10]{
    \includegraphics[width=1.4 cm]{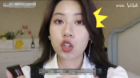}
    \label{fig:s10}
    }  
    \subfigure[Sid: 11]{
    \includegraphics[width=1.4 cm]{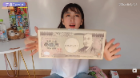}
    \label{fig:s11}
    }
    \subfigure[Sid: 12]{
    \includegraphics[width=1.4 cm]{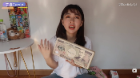}
    \label{fig:s12}
    }
    \subfigure[Sid: 13]{
    \includegraphics[width=1.4 cm]{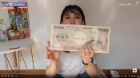}
    \label{fig:s13}
    }
    \subfigure[Sid: 14]{
    \includegraphics[width=1.4 cm]{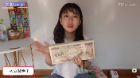}
    \label{fig:s14}
    }
    \subfigure[Sid: 15]{
    \includegraphics[width=1.4 cm]{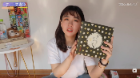}
    \label{fig:s15}
    }  
   \vspace{-2ex} \caption{\small Video segment samples.}  \vspace{-1ex}  
    \label{fig:VideoSegSamples}
\end{figure}
\begin{table}[tb] 
\setlength\tabcolsep{2.pt}
    \begin{center}
        \begin{tabular}{|l|cc|cc|cc|cc|cc|cc|c|}
        \hline
         $S_i$   & \multicolumn{2}{|c|}{LTR}& \multicolumn{2}{|c|}{VEC}& \multicolumn{2}{|c|}{LSTM} 
  &  \multicolumn{2}{|c|}{RTFM}&\multicolumn{2}{|c|}{CLSTM-S}&\multicolumn{2}{|c|}{CLSTM} & L$_g$.\\
           &Score&L$_p$&Score&L$_p$&Score &L$_p$&Score&L$_p$&Score&L$_p$&Score&L$_p$& \\
            \hline
           1& 0.013& 0 & 0.043& 0& 0.022& 0& 0.028& 0& 0.021& 0& 0.055 &0 &0\\
           2&   0.033&  0 &0.040 & 0 &  0.118 & \textbf{\underline{1}}&0.070& 0& 0.019 & 0&0.089  &\textbf{\underline{1}} &0\\
           3& 0.110 &   1  &0.045 & 1& 0.054& 1& 0.127& 1&  0.141 & 1& 0.124 & 1&1\\
           4&  0.107 &  1   & 0.224  & 1& 0.021 & \textbf{\underline{0}}& 0.431& 1& 0.023 & \textbf{\underline{0}}& 0.103 & 1&1\\
           5& 0.058&   \textbf{\underline{0}}  & 0.028  & \textbf{\underline{0}}&  0.146 &  1& 0.126& 1&  0.065 & 1& 0.271 & 1&1\\
           6& 0.060 &  0   &0.004  & 0& 0.018  & 0& 0.012& 0& 0.032& 0&0.053 &0 &0\\
           7& 0.147  &  \textbf{\underline{1}}   & 0.068 & \textbf{\underline{1}} & 0.031& 0& 0.016& 0& 0.030 &0 & 0.057&0 &0\\
           8& 0.204 &  1 & 0.082  & 1& 0.189  & 1& 0.127& 1&  0.082 & 1& 0.263 & 1&1\\
           9&  0.110 &  1  &  0.108& 1&0.213 & 1& 0.022& \textbf{\underline{0}} &0.143 & 1& 0.145  &1 &1\\
           10& 0.237&  1 & 0.097 & 1&  0.109 & 1&0.290 & 1& 0.087 & 1&0.090 & 1&1\\
         11& 0.237&  1 & 0.097 & 1&  0.092 & 1&0.290 & 1& 0.058 & 1&0.083 & 1&1\\
         12& 0.341&  \textbf{\underline{1}}& 0.008 & 0&  0.029 & 0&0.290 & \textbf{\underline{1}}& 0.008 & 0&0.045 & 0&0\\
         13& 0.113&  \textbf{\underline{1}} & 0.023 & 0&  0.119 & \textbf{\underline{1}}&0.064 & 0& 0.017 & 0&0.032 & 0&0\\
         14& 0.031&   0& 0.037 & 0&  0.023 & 0&0.035 & 0& 0.035 & 0&0.069 & 0&0\\
         15& 0.870&  \textbf{\underline{1}} & 0.103 & \textbf{\underline{1}}&  0.153 & \textbf{\underline{1}}&0.136 & \textbf{\underline{1}}& 0.033 & 0&0.055 & 0&0\\
           \hline
        \end{tabular}
    \end{center}\vspace{-2ex}
    \caption{\small Anomaly detection results of video segment samples.}\vspace{-6ex}
    \label{tab:detectRst}
\end{table}
\section{Conclusion}\label{sec:conclusion}

This paper proposes a novel framework AOVLIS for anomaly detection over social video live streaming. We first propose a novel CLSTM for video segment series. Then we propose a weighted RE scoring for anomaly detection. Moreover, we propose to dynamically update our model incrementally. 
To speed up the detection efficiency, we propose a new ADG-based dimensionality reduction for action recognition features together with a non-trivial lower bound distance in the reduced group space to filter out the false alarms without false dismissals. An adaptive optimisation strategy is proposed to select a bound function for efficient anomaly candidate filtering. Extensive tests prove the high efficacy of  AOVLIS.

\eat{
\section*{Acknowledgment}
The work is partially supported by ARC Discovery Projects (DP200101175) and CSIRO Data61 Grant.
}

\end{document}